\documentclass[11pt]{article}

\usepackage[margin=1in]{geometry}
\usepackage{amsthm}
\usepackage{amssymb}
\usepackage{amsmath}
\usepackage{graphicx}
\usepackage{url}
\usepackage{color}
\usepackage{multirow}
\usepackage{tabularx}

\newtheorem{theorem}{Theorem}

\newtheorem{lemma}[theorem]{Lemma}

\theoremstyle{definition}
\newtheorem{problem}[theorem]{Problem}

\newtheorem{definition}[theorem]{Definition}

\newcommand{\sqz}{\operatorname{sqz}}
\newcommand{\dual}{\operatorname{dual}}
\newcommand{\val}{\operatorname{val}}
\newcommand{\sspan}{\operatorname{span}}

\newcommand{\tr}{\operatorname{tr}}

\title{SqueezeFit: Label-aware dimensionality reduction\\by semidefinite programming}
\author{Culver McWhirter\thanks{The Ohio State University, Columbus, OH 43210} \qquad Dustin G.\ Mixon\footnotemark[1] \qquad Soledad Villar\thanks{New York University, New York, NY 10003}}
\date{}

\begin{document}
\maketitle

\begin{abstract}
Given labeled points in a high-dimensional vector space, we seek a low-dimensional subspace such that projecting onto this subspace maintains some prescribed distance between points of differing labels.
Intended applications include compressive classification.
Taking inspiration from large margin nearest neighbor classification, this paper introduces a semidefinite relaxation of this problem.
Unlike its predecessors, this relaxation is amenable to theoretical analysis, allowing us to provably recover a planted projection operator from the data.
\end{abstract}

\section{Introduction}

The last decade of sampling theory has transformed the way we reconstruct signals from measurements.
For example, the now-established theory of compressed sensing allows one to reconstruct a signal from a number of random linear measurements that is proportional to the complexity of that signal~\cite{Donoho:06,CandesRT:06,FoucartR:13}, potentially speeding up MRI scans by a factor of five~\cite{LustigDP:07}.
This theory has since transferred to the setting of nonlinear measurements in the context of phase retrieval~\cite{CandesSV:13,CandesESV:15}, leading to new algorithms for coherent diffractive imaging~\cite{SidorenkoEtal:15}.
Today, we witness major technological advances in machine learning, where neural networks have recently achieved unprecedented performance in image classification and elsewhere~\cite{KrizhevskySH:12,SilverEtal:16}.
This motivates another fundamental problem for sampling theory:
\begin{center}
\textit{How many samples are necessary to enable signal classification?}
\end{center}
For instance, why waste time collecting enough samples to completely reconstruct a given signal if you only need to detect whether the signal contains an anomaly?

This different approach to sampling is known as \textbf{compressive classification}.
While the idea has been around since 2007, to date, only three works provide theory to derive sampling rates for compressive classification.
First, \cite{DavenportEtal:07} considered the case where each class is a low-dimensional manifold.
Much later, \cite{ReboredoRCR:16} compressively classified mixtures of Gaussians of low-rank covariance, and then \cite{BandeiraMR:17} derived sampling rates for random projection to maintain separation between full-dimensional ellipsoids.
Overall, these works assumed that the classes follow a specific model (be it manifolds, Gaussians or ellipsoids), and then derived conditions under which a good projection exists.
The present work takes a dual approach:
We assume that compressive classification is possible, meaning there exists a planted low-rank projection that facilitates classification, and the task is to derive conditions on the classes for which finding that projection is feasible:


\begin{problem}[projection factor recovery]
\label{prob.pfr}
Let $\Pi$ denote orthogonal projection onto some unknown subspace $T\subseteq\mathbb{R}^d$ of some unknown dimension.
What conditions on $f\colon T\to[k]:=\{1,\ldots,k\}$ and $\mathcal{X}\subseteq\mathbb{R}^d$ enable exact or approximate recovery of $\Pi$ from data of the form $\{(x,f(\Pi x))\}_{x\in\mathcal{X}}$?
\end{problem}

In words, we assume the classification function factors through some unknown orthogonal projection operator $\Pi$, and the objective is to reconstruct $\Pi$.
Once we find $\Pi$ of rank $r$, then we may write $\Pi=A^\top A$ for some $r\times d$ sensing matrix $A$, and then $Ax$ determines the classification $f(\Pi x)$ of $x$ despite using only $r\ll d$ samples.
Here and throughout, we consider a sequence of data $\mathcal{D}=\{(x_i,y_i)\}_{i\in\mathcal{I}}$ in $\mathbb{R}^d\times[k]$ and denote $\mathcal{Z}(\mathcal{D}):=\{x_i-x_j:i,j\in\mathcal{I},y_i\neq y_j\}$.
The following program finds the best orthogonal projection for our purposes:
\begin{equation}
\label{eq.min rank program}
\text{minimize}
\quad
\operatorname{rank}\Pi
\quad
\text{subject to}
\quad
\|\Pi z\|\geq\Delta~~\forall z\in\mathcal{Z}(\mathcal{D}),
\quad
\Pi^\top=\Pi,
\quad
\Pi^2=\Pi
\end{equation}
Here, $\Pi$ is the decision variable, whereas $\Delta>0$ is a parameter that prescribes a desired minimum distance between projected points $\Pi x_i$ and $\Pi x_j$ with differing labels.
This parameter reflects a fundamental tension in compressive classification:
We want $\Delta$ to be large so as to enable classification, but we also want $\operatorname{rank}\Pi$ to be small so that this classification is compressive.
Since it is not clear how to tractably implement \eqref{eq.min rank program}, we consider a convex relaxation:
\begin{equation}
\text{minimize}
\quad
\operatorname{tr}M
\quad
\text{subject to}
\quad
z^\top Mz\geq\Delta^2~~\forall z\in\mathcal{Z}(\mathcal{D}),
\quad
0\preceq M\preceq I
\tag{$\sqz(\mathcal{D},\Delta)$}
\end{equation}
We refer to this program as \textbf{SqueezeFit}.
If $\mathcal{Z}(\mathcal{D})$ is finite, then $\sqz(\mathcal{D},\Delta)$ is a semidefinite program, otherwise $\sqz(\mathcal{D},\Delta)$ is a semi-infinite program \cite{semi-infinite}.
In either case, the minimum exists whenever $\sqz(\mathcal{D},\Delta)$ is feasible by the extreme value theorem.
As Figure~\ref{fig.comparison} illustrates, SqueezeFit is well suited for projection factor recovery.

\begin{figure}[t]
\begin{center}
\includegraphics[width=\textwidth]{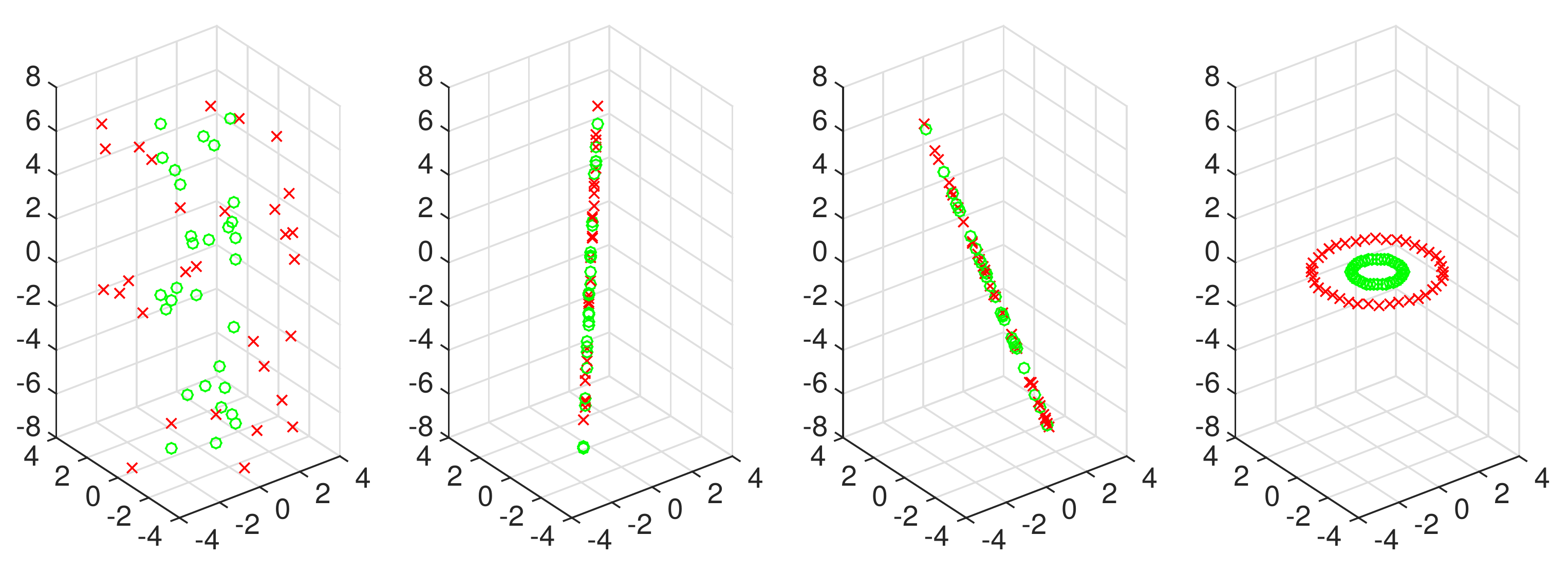}
\end{center}
\caption{\label{fig.comparison}
\textbf{(far left)}
Plot of $60$ data points in $\mathbb{R}^3$, half in one class, half in another.
These points were drawn according to a random model with an unknown planted projection factor (as in Problem~\ref{prob.pfr}).
\textbf{(middle left)}
Principal component analysis (PCA) suggests one-dimensional structure in the data.
Projecting onto this subspace (which was identified without regard for the points' classes) results in an undesirable mixture of the classes.
\textbf{(middle right)}
Unlike PCA, linear discriminant analysis (LDA) actually considers which class each point belongs to.
Since there are two classes, the result is projection onto a $1$-dimensional subspace, obtained by applying the classes' inverse covariance matrix to the difference of class centroids.
Unfortunately, the result is again an unhelpful mixture of classes.
\textbf{(far right)}
Unlike PCA and LDA, SqueezeFit finds a low-rank projection that maintains some amount of distance between points from different classes.
The resulting projection is a close approximation to the planted projection factor.
See Section~\ref{sec.pfr with sf} for theoretical guarantees that help explain this behavior.
}
\end{figure}

When formulating SqueezeFit, the authors took inspiration from the large margin nearest neighbor (LMNN) algorithm~\cite{WeinbergerS:09}, which finds the $d\times d$ matrix $M\succeq0$ such that $\{(M^{1/2}x_i,y_i)\}_{i\in\mathcal{I}}$ is best conditioned for $k$-nearest neighbor classification in the Euclidean distance (unlike above, $k$ does not correspond to the number of classes here).
To accomplish this, LMNN first identifies for each $x_i$, the $k$ closest $x_j$ such that $y_j=y_i$; these are called \textit{target neighbors}.
Next, $x_l$ is called an \textit{impostor} of $x_i$ under $M$ if $x_i$ has a target neighbor $x_j$ such that $\|M^{1/2}(x_i-x_l)\|^2\leq\|M^{1/2}(x_i-x_j)\|^2+1$.
Intuitively, $\{(M^{1/2}x_i,y_i)\}_{i\in\mathcal{I}}$ is well conditioned for $k$-nearest neighbors if the target neighbors are all close to each other, and the number of impostors is small.
To this end, LMNN uses a semidefinite program to find the $M\succeq0$ that simultaneously optimizes these conflicting objectives.

While LMNN has proven to be an effective tool for metric learning, there is currently a dearth of theory to explain its performance.
By contrast, our formulation of SqueezeFit is particularly amenable to theoretical analysis, which we credit to two features:
First, we do not require a pre-processing step to define target neighbors, thereby isolating how our algorithm depends on the data.
In exchange for this lack of pre-processing, we accept the hyperparameter $\Delta$ to provide some notion of ``impostor.''
Second, SqueezeFit includes the identity constraint $M\preceq I$, which proves particularly valuable to the theory.
For example, the identity constraint plays a key role in the proof that, if $M$ is SqueezeFit-optimal for $\{(x_i,y_i)\}_{i\in\mathcal{I}}$, then the only SqueezeFit-optimal operator for $\{(M^{1/2}x_i,y_i)\}_{i\in\mathcal{I}}$ is orthogonal projection onto $\sspan\{M^{1/2}x_i\}_{i\in\mathcal{I}}$ (see Theorem~\ref{thm.sqz as proj}).
In words, \textit{you don't need to squeeze your data more than once}.

In the next section, we study some of the important geometric features of SqueezeFit, and then we use these features to analyze strong duality.
This analysis is a prerequisite for Section~\ref{sec.pfr with sf}, where we derive conditions under which SqueezeFit successfully performs projection factor recovery.
SqueezeFit also performs well in practice, which we illustrate in Section~4 with an assortment of numerical experiments.
We conclude in Section~5 with a discussion of various open questions and opportunities for future work.

\section{Model-free theory}

\subsection{The geometry of SqueezeFit}

Throughout, $\mathcal{D}$ denotes a sequence in $\mathbb{R}^d\times[k]$ without mention.

\begin{definition}
We say $\mathcal{D}=\{(x_i,y_i)\}_{i\in\mathcal{I}}$ is \textbf{$\Delta$-fixed} if there exists $M\in \arg\sqz(\mathcal{D},\Delta)$ such that $M^{1/2}x_i=x_i$ for every $i\in\mathcal{I}$.
\end{definition}

\begin{lemma}
\label{lem.ep.vectors}
Pick any $\mathcal{D}=\{(x_i,y_i)\}_{i\in\mathcal{I}}$.
\begin{itemize}
\item[(i)]
$\mathcal{D}$ is $\Delta$-fixed if and only if orthogonal projection onto $\operatorname{span}\{x_i\}_{i\in\mathcal{I}}$ is the unique member of $\arg\sqz(\mathcal{D},\Delta)$.
\item[(ii)]
If $\mathcal{D}$ is $\Delta$-fixed, then $\operatorname{span}\mathcal{Z}(\mathcal{D})=\operatorname{span}\{x_i\}_{i\in\mathcal{I}}$.
\end{itemize}
\end{lemma}

\begin{proof}
(i)
First, ($\Leftarrow$) is immediate.
For ($\Rightarrow$), we have by assumption that there exists $M\in\arg\sqz(\mathcal{D},\Delta)$ with a leading eigenvalue of $1$ whose eigenspace contains every $x_i$.
Let $\Pi$ denote orthogonal projection onto $\sspan\{x_i\}_{i\in\mathcal{I}}$.
Then we may write $M=\Pi+\Gamma$ for some $\Gamma$ satisfying $0\preceq\Gamma\preceq I$ and $\Gamma\Pi=\Pi\Gamma=0$.
Note that $\Pi$ is feasible in $\sqz(\mathcal{D},\Delta)$ since $0\preceq\Pi\preceq I$ and 
\[
z^\top \Pi z=(\Pi z)^\top(\Pi z)=z^\top z\geq z^\top Mz\geq\Delta^2
\]
for every $z\in\mathcal{Z}(\mathcal{D})$.
Finally, we must have $M=\Pi$, i.e., $\Gamma=0$, since otherwise $\tr\Pi < \tr\Pi + \tr\Gamma = \tr M$, thereby violating the assumption that $M\in\arg\sqz(\mathcal{D},\Delta)$.

(ii)
By (i), $\Pi$ is feasible in $\sqz(\mathcal{D},\Delta)$.
Let $\Pi_\mathcal{Z}$ denote orthogonal projection onto $\operatorname{span}\mathcal{Z}(\mathcal{D})$.
Then $0\preceq\Pi_\mathcal{Z}\preceq I$ and 
\[
z^\top \Pi_\mathcal{Z} z=(\Pi_\mathcal{Z} z)^\top(\Pi_\mathcal{Z} z)=z^\top z\geq z^\top \Pi z\geq\Delta^2
\]
for every $z\in\mathcal{Z}(\mathcal{D})$, and so $\Pi_\mathcal{Z}$ is also feasible in $\sqz(\mathcal{D},\Delta)$.
Since $\Pi\in\arg\sqz(\mathcal{D},\Delta)$ by (i), we then have
\[
\operatorname{dim}\sspan\{x_i\}_{i\in\mathcal{I}}
=\tr\Pi
\leq\tr\Pi_\mathcal{Z}
=\operatorname{dim}\sspan\mathcal{Z}(\mathcal{D}).
\]
The definition of $\mathcal{Z}(\mathcal{D})$ implies that $\operatorname{span}\mathcal{Z}(\mathcal{D})\subseteq\operatorname{span}\{x_i\}_{i\in\mathcal{I}}$, and so the above dimension count gives the desired equality.
\end{proof}

\begin{theorem}
\label{thm.sqz as proj}
Given $\mathcal{D}=\{(x_i,y_i)\}_{i\in\mathcal{I}}$, then $\{(M^{1/2}x_i,y_i)\}_{i\in\mathcal{I}}$ is $\Delta$-fixed for every $M\in\arg\sqz(\mathcal{D},\Delta)$.
\end{theorem}

\begin{proof}
Pick $M\in\arg\sqz(\mathcal{D},\Delta)$, put $\mathcal{D}'=\{(M^{1/2}x_i,y_i)\}_{i\in\mathcal{I}}$, and pick $N\in\arg\sqz(\mathcal{D}',\Delta)$.
We claim that $M_0:=(M^{1/2})^\top NM^{1/2}$ is feasible in $\sqz(\mathcal{D},\Delta)$.
First, we have $\mathcal{Z}(\mathcal{D}')=M^{1/2}\mathcal{Z}(\mathcal{D})$, and so the feasibility of $N$ in $\sqz(\mathcal{D}',\Delta)$ implies
\[
z^\top M_0z
=(M^{1/2}z)^\top N (M^{1/2}z)
\geq \Delta^2
\]
for every $z\in\mathcal{Z}(\mathcal{D})$.
Similarly, $0\preceq M_0\preceq I$ follows from the facts that $0\preceq N\preceq I$ and $M\preceq I$:
\begin{align*}
x^\top M_0x
&=(M^{1/2}x)^\top N (M^{1/2}x)\geq0,\\
x^\top M_0x
&=(M^{1/2}x)^\top N (M^{1/2}x)
\leq (M^{1/2}x)^\top (M^{1/2}x)
=x^\top Mx
\leq x^\top x
\end{align*}
for every $x\in\mathbb{R}^d$.
Overall, we indeed have that $M_0$ is feasible in $\sqz(\mathcal{D},\Delta)$.

Next, let $\alpha_1\geq\cdots\geq\alpha_d$ and $\beta_1\geq\cdots\beta_d$ denote the eigenvalues of $M$ and $N$, respectively.
Then the von Neumann trace inequality gives
\[
\sum_{j=1}^d\alpha_j
=\tr M
\leq \tr M_0
=\tr(MN)
\leq\sum_{j=1}^d \alpha_j\beta_j
\leq \sum_{j=1}^d \alpha_j,
\]
where the last inequality uses the facts that $\alpha_j\geq0$ and $\beta_j\leq 1$ for every $j$.
Considering the far left- and right-hand sides, all inequalities are necessarily equalities.
Equality in the von Neumann trace inequality implies that $M$ and $N$ are simultaneously unitarily diagonalizable, while equality in the last inequality implies that $\beta_j=1$ whenever $\alpha_j\neq0$.
As such, $N^{1/2}$ fixes the column space of $M^{1/2}$, meaning $\mathcal{D}'=\{(M^{1/2}x_i,y_i)\}_{i\in\mathcal{I}}$ satisfies the definition of $\Delta$-fixed, as desired.
\end{proof}

\begin{definition}
The \textbf{contact vectors} of $\mathcal{D}$ are the shortest vectors in $\mathcal{Z}(\mathcal{D})$, when they exist.
\end{definition}

\begin{lemma}
\label{lem.fixed contact length}
If $\mathcal{D}$ is $\Delta$-fixed, then its contact vectors have length $\Delta$, provided they exist.
\end{lemma}

\begin{proof}
Suppose $\mathcal{D}$ has a contact vector.
Since $\mathcal{D}$ is $\Delta$-fixed, $\sqz(\mathcal{D},\Delta)$ is feasible, and so this contact vector has length $\geq\Delta$.
Suppose the contact vector has length $\Delta_0>\Delta$, put $\alpha=(\Delta/\Delta_0)^2<1$, and select $\Pi\in\arg\sqz(\mathcal{D},\Delta)$ such that $\Pi^{1/2}x_i=x_i$ for every $i\in\mathcal{I}$.
By Lemma~\ref{lem.ep.vectors}(i), $\Pi$ is orthogonal projection onto $\sspan\{x_i\}_{i\in\mathcal{I}}$.
We will show that $M:=\alpha\Pi$ is feasible in $\sqz(\mathcal{D},\Delta)$ with smaller trace than $\Pi$, contradicting the fact that $\Pi$ is optimal in $\sqz(\mathcal{D},\Delta)$.
Since $\alpha\in(0,1)$, we have $0\preceq M\preceq I$.
Next, every $z\in\mathcal{Z}(\mathcal{D})$ satisfies
\[
z^\top Mz
=\alpha z^\top\Pi z
=\alpha z^\top z
\geq\Delta^2,
\]
where the last equality applies Lemma~\ref{lem.ep.vectors}(i) and the inequality uses the definition of $\alpha$.
Finally, $\tr M=\alpha \tr\Pi <\tr\Pi$, producing the desired contradiction.
\end{proof}

\begin{theorem}
\label{thm.contact fixed}
If $\mathcal{D}=\{(x_i,y_i)\}_{i\in\mathcal{I}}$ has contact vectors of length $\Delta$ that span $\sspan\{x_i\}_{i\in\mathcal{I}}$, then  $\mathcal{D}$ is $\Delta$-fixed.
Furthermore, the converse holds when $|\mathcal{Z}(\mathcal{D})|<\infty$.
\end{theorem}

See Figure~\ref{figure} for (necessarily infinite) examples in which the converse fails to hold.

\begin{figure}
\begin{center}
\includegraphics[height=0.35\textwidth]{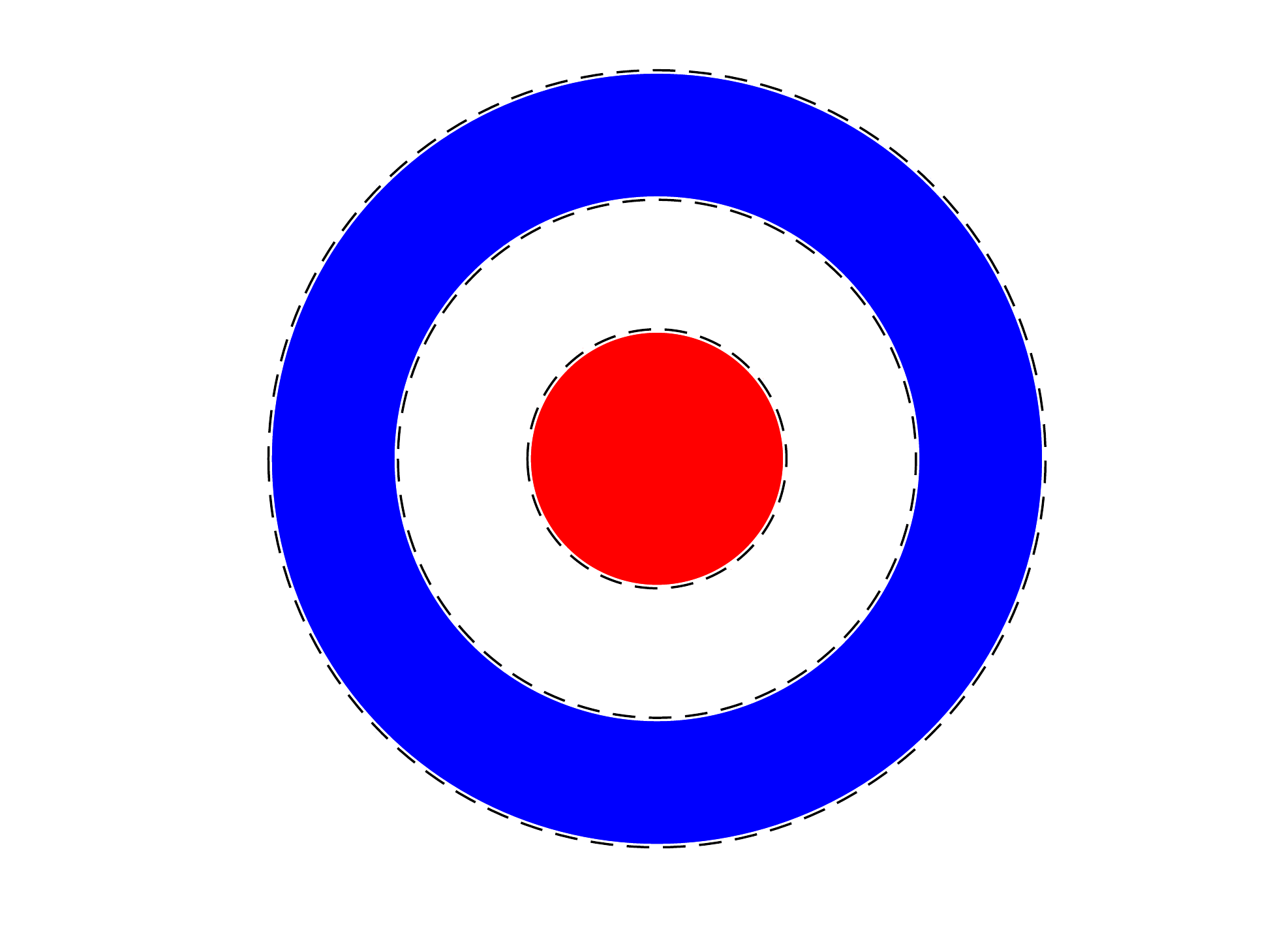}
\includegraphics[height=0.35\textwidth]{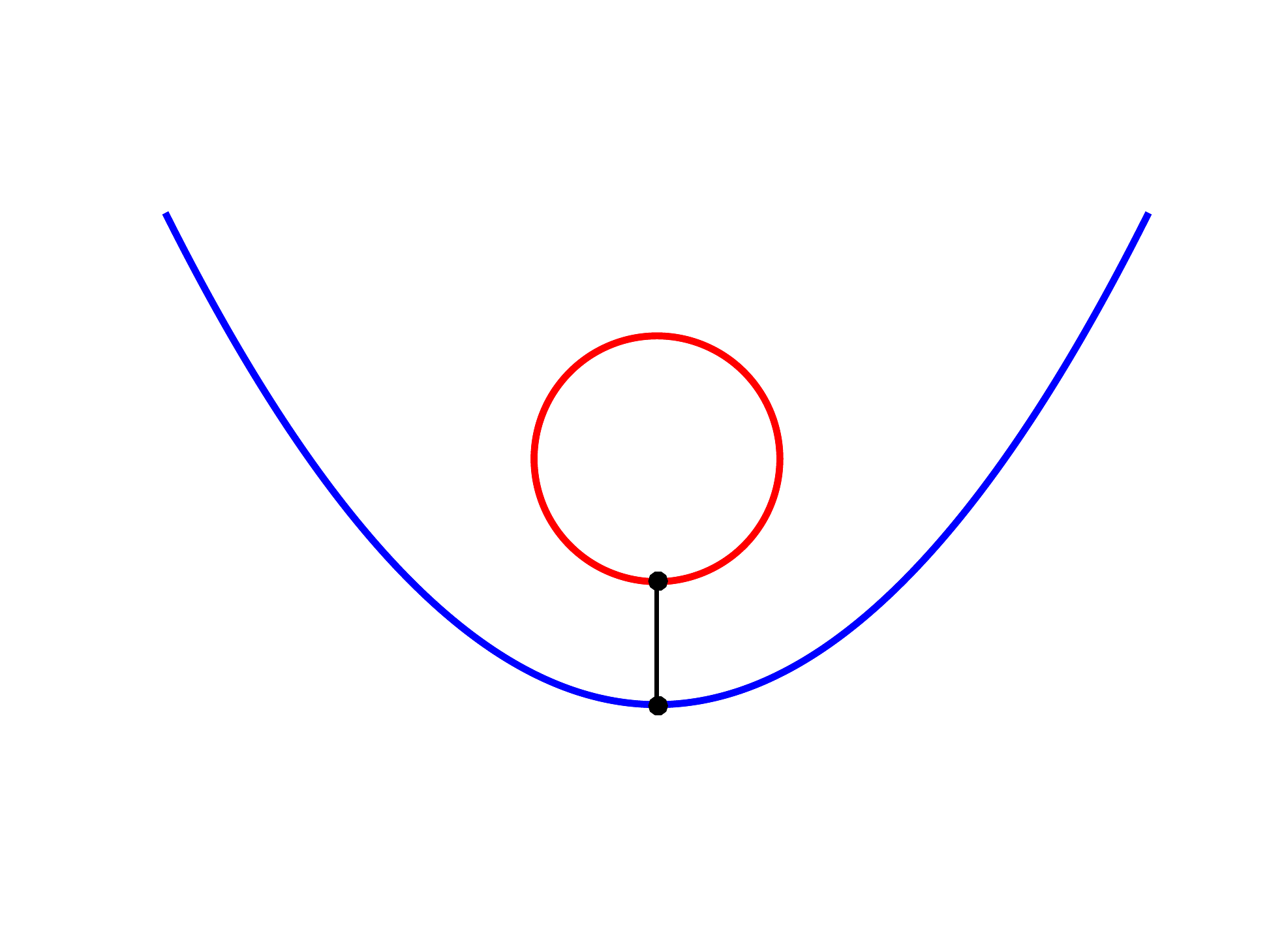}
\end{center}
\caption{\label{figure}
Examples of $\Delta$-fixed data that fail to admit a spanning set of contact vectors.
Both examples exhibit $k=2$ classes in $\mathbb{R}^2$.
On the left, the classes are open sets, and they admit no contact vectors.
On the right, the classes are compact sets, and the only contact vectors are $(0,\pm1)$.
These examples illustrate why Theorem~\ref{thm.contact fixed} requires $|\mathcal{Z}(\mathcal{D})|<\infty$ for the converse to hold.
By Theorem~\ref{thm.strong duality fixed contact span}, SqueezeFit fails to admit a Haar dual certificate in these examples.}
\end{figure}

\begin{proof}[Proof of Theorem~\ref{thm.contact fixed}]
Pick any $M\in\arg\sqz(\mathcal{D},\Delta)$.
Then for every contact vector $z$ of $\mathcal{D}$, we have
\[
\Delta^2
\leq z^\top Mz
\leq z^\top z
= \Delta^2.
\]
Considering the far left- and right-hand sides, all inequalities are necessarily equalities.
In particular, equality in the second inequality combined with $M\preceq I$ implies that $M$ has a leading eigenvalue of $1$ whose eigenspace contains every contact vector, and therefore every $x_i$ (by assumption).
This implies that $\mathcal{D}$ satisfies the definition of $\Delta$-fixed.

For the converse, suppose $|\mathcal{Z}(\mathcal{D})|<\infty$ and $\mathcal{D}$ is $\Delta$-fixed.
Since $|\mathcal{Z}(\mathcal{D})|<\infty$, $\mathcal{D}$ necessarily has a contact vector.
By Lemma~\ref{lem.fixed contact length}, the contact vectors of $\mathcal{D}$ necessarily have length $\Delta$.
Let $S$ denote the span of these contact vectors.
We seek to prove $S=\sspan\{x_i\}_{i\in\mathcal{I}}$, and by Lemma~\ref{lem.ep.vectors}(ii), it is equivalent to show $S=\sspan\mathcal{Z}(\mathcal{D})$.
Since $S\subseteq\sspan\mathcal{Z}(\mathcal{D})$, it suffices to show $\sspan\mathcal{Z}(\mathcal{D})\subseteq S$.
To this end, consider the set $A:=\{\|z\|:z\in\mathcal{Z}(\mathcal{D})\setminus S\}\subseteq(\Delta,\infty)$.
If $A$ is empty, then we are done, since this implies $\mathcal{Z}(\mathcal{D})\subseteq S$.
Otherwise, since $|\mathcal{Z}(\mathcal{D})|<\infty$, we have $\alpha:=(\Delta/(\min A))^2<1$.
Let $\Pi_S$ and $\Pi_T$ denote orthogonal projection onto $S$ and $T:=S^\perp\cap\sspan\{x_i\}_{i\in\mathcal{I}}$, respectively, and select $\Pi\in\arg\sqz(\mathcal{D},\Delta)$.
We will show that $M:=\Pi_S+\alpha\Pi_T$ is feasible in $\sqz(\mathcal{D},\Delta)$ with smaller trace than $\Pi$, contradicting the fact that $\Pi$ is optimal in $\sqz(\mathcal{D},\Delta)$.

Since $\alpha\in(0,1)$, we have $0\preceq M\preceq I$.
Next, we will verify $z^\top Mz\geq\Delta^2$ for every $z\in\mathcal{Z}(\mathcal{D})$.
For $z\in\mathcal{Z}(\mathcal{D})\cap S$, we have
\[
z^\top Mz
=z^\top(\Pi_S+\alpha\Pi_T)z
=z^\top \Pi_Sz
=z^\top z
\geq\Delta^2.
\]
Meanwhile, for $z\in\mathcal{Z}(\mathcal{D})\setminus S$, we have 
\[
z^\top Mz
=z^\top(\Pi_S+\alpha\Pi_T)z
\geq\alpha z^\top(\Pi_S+\Pi_T) z
\stackrel{\text{(a)}}{=}\alpha z^\top\Pi z
=\alpha\|\Pi z\|^2
\stackrel{\text{(b)}}{=}\alpha\|z\|^2
\geq\Delta^2,
\]
where (a) and (b) follow from Lemma~\ref{lem.ep.vectors}(i), and the final inequality follows from the definition of $\alpha$.
Overall, $M$ is feasible in $\sqz(\mathcal{D},\Delta)$, and yet
\[
\tr M
=\tr \Pi_S + \alpha \tr \Pi_T
<\tr \Pi_S + \tr \Pi_T
=\tr(\Pi_S+\Pi_T)
=\tr\Pi,
\]
where the last equality again follows from Lemma~\ref{lem.ep.vectors}(i).
This is the desired contradiction.
\end{proof}

\subsection{Conditions for strong duality}

We follow \cite{semi-infinite} to find the Haar dual program of $\sqz(\mathcal{D},\Delta)$ in the case where $\mathcal{Z}(\mathcal{D})$ is infinite:
\begin{align}
\tag{$\dual(\mathcal{D},\Delta)$}
\sup \quad & \Delta^2\sum_{z\in\mathcal{Z}(\mathcal{D})}\gamma(z)-\tr Y\\ \nonumber\text{subject to} \quad & \sum_{z\in\mathcal{Z}(\mathcal{D})}\gamma(z)zz^\top-Y\preceq I, \quad Y\succeq0, \quad \gamma\geq0, \quad |\operatorname{supp}(\gamma)|<\infty
\end{align}
Here, the decision variables are $\gamma\colon\mathcal{Z}(\mathcal{D})\to\mathbb{R}$ and $Y\in\mathbb{R}^{d\times d}$.
The above program reduces to the dual semidefinite program when $\mathcal{Z}(\mathcal{D})$ is finite.
In either case, weak duality gives that $\val\sqz(\mathcal{D},\Delta)\geq\val\dual(\mathcal{D},\Delta)$ when $\sqz(\mathcal{D},\Delta)$ is feasible.
We are interested in when $\sqz(\mathcal{D},\Delta)$ admits a \textbf{Haar dual certificate}, that is, a maximizer of $\dual(\mathcal{D},\Delta)$.
Indeed, a Haar dual certificate certifies the optimality of a given optimal point in $\sqz(\mathcal{D},\Delta)$ by witnessing equality in weak duality.
In the finite case, we can prove the existence of a Haar dual certificate by manipulating $\sqz(\mathcal{D},\Delta)$ and applying Slater's condition:

\begin{theorem}
\label{thm.finite sd}
Suppose $|\mathcal{Z}(\mathcal{D})|<\infty$ and $\sqz(\mathcal{D},\Delta)$ is feasible.
Then $\sqz(\mathcal{D},\Delta)$ admits a Haar dual certificate.
\end{theorem}

\begin{proof}
Put $\mathcal{Z}_1=\{z\in\mathcal{Z}(\mathcal{D}):\|z\|=\Delta\}$, $\mathcal{Z}_2=\mathcal{Z}(\mathcal{D})\setminus\mathcal{Z}_1$, and $T=\sspan\{z\}_{z\in\mathcal{Z}_1}$, and consider the related semidefinite program:
\begin{align}
\label{eq.newprimal}
\min\quad & \tr(\Pi_{T^\perp}X\Pi_{T^\perp})\\
\nonumber
\text{subject to}\quad & z^\top \Pi_{T^\perp}X\Pi_{T^\perp} z\geq \Delta^2-\|\Pi_Tz\|^2~~ \forall z\in\mathcal{Z}_2, \quad \Pi_TX\Pi_T=0, \quad 0\preceq X\preceq I
\end{align}
Importantly, every feasible point $X$ of \eqref{eq.newprimal} can be transformed into a feasible point $M=\Pi_{T^\perp} X\Pi_{T^\perp}+\Pi_T$ of $\sqz(\mathcal{D},\Delta)$.
Indeed, if $z\in\mathcal{Z}_2$, then
\[
z^\top(\Pi_{T^\perp} X\Pi_{T^\perp}+\Pi_T)z
=z^\top\Pi_{T^\perp}X\Pi_{T^\perp}z+\|\Pi_Tz\|^2
\geq \epsilon,
\]
while if $z\in\mathcal{Z}_1$, then $z\in T$, and so $z^\top(\Pi_{T^\perp} X\Pi_{T^\perp}+\Pi_T)z=\|z\|^2=\epsilon$.
Furthermore, $0\preceq \Pi_{T^\perp} X\Pi_{T^\perp}+\Pi_T\preceq \Pi_{T^\perp}+\Pi_T=I$.

Next, we demonstrate that the related program \eqref{eq.newprimal} satisfies strong duality by Slater's theorem.
To see this, define $\alpha=\frac{1}{2}+\frac{1}{2\Delta^2}\min\{\|z\|^2:z\in\mathcal{Z}_2\}$.
Then $1<\alpha<\|z\|^2/\Delta^2$ for every $z\in\mathcal{Z}_2$.
Consider $X_0=\frac{1}{\alpha}\Pi_{T^\perp}$.
Then for every $z\in\mathcal{Z}_2$, we have
\[
z^\top\Pi_{T^\perp} X_0\Pi_{T^\perp}z
=\frac{1}{\alpha}\|\Pi_{T^\perp}z\|^2
=\frac{1}{\alpha}\|z\|^2-\frac{1}{\alpha}\|\Pi_{T}z\|^2
>\Delta^2-\|\Pi_{T}z\|^2.
\]
Furthermore, $X_0$ lies in the relative interior of $\{X:\Pi_{T}X\Pi_{T}=0,0\preceq X\preceq I\}$.
Overall, $X_0$ lies in the relative interior of the feasibility region of \eqref{eq.newprimal}, and so Slater's theorem~\cite{BoydV:04} implies that the value of \eqref{eq.newprimal} equals the value of its dual:
\begin{align}
\label{eq.newdual}
\max\quad&\sum_{z\in\mathcal{Z}_2}(\Delta^2-\|\Pi_Tz\|^2)\tilde{\gamma}(z)-\operatorname{tr}\tilde{Y}\\
\nonumber
\text{subject to}\quad&\sum_{z\in\mathcal{Z}_2}\tilde{\gamma}(z)(\Pi_{T^\perp}z)(\Pi_{T^\perp}z)^\top-\tilde{Y}\preceq \Pi_{T^\perp}, \quad \Pi_T\tilde{Y}\Pi_T=0,\quad \tilde{Y}\succeq0, \quad \tilde{\gamma}\geq0
\end{align}
Here, the decision variables are $\tilde\gamma\colon\mathcal{Z}_2\to\mathbb{R}$ and $\tilde{Y}\in\mathbb{R}^{d\times d}$.

Next, we demonstrate how to transform every feasible point $(\tilde{\gamma},\tilde{Y})$ of \eqref{eq.newdual} into a feasible point $(\gamma,Y)$ of $\dual(\mathcal{D},\Delta)$.
To this end, define $Z=\sum_{z\in\mathcal{Z}_1}zz^\top$, and let $\lambda$ denote the smallest nonzero eigenvalue of $Z$.
Then we take
\[
\gamma(z)
:=\left\{\begin{array}{ll}
1/\lambda&\text{if }z\in\mathcal{Z}_1\\
\tilde\gamma(z)&\text{if }z\in\mathcal{Z}_2
\end{array}\right\},
\qquad
Y:=\sum_{z\in\mathcal{Z}_2}\tilde{\gamma}(z)(\Pi_Tz)(\Pi_Tz)^\top+\frac{1}{\lambda}Z-\Pi_T+\tilde{Y}.
\]
Then we immediately have $\gamma\geq0$ and $Y\succeq0$.
Furthermore,
\[
R
:=I-\sum_{z\in\mathcal{Z}(\mathcal{D})}\gamma(z)zz^\top+Y
=\Pi_{T^\perp}-\sum_{z\in\mathcal{Z}_2}\tilde{\gamma}(z)\Big(zz^\top-(\Pi_Tz)(\Pi_Tz)^\top\Big)+\tilde{Y}
\]
satisfies $\Pi_TR\Pi_T=0$ and $\Pi_{T^\perp}R\Pi_{T^\perp}\succeq0$ by feasibility in \eqref{eq.newdual}, and so $R\succeq0$, as desired.

By assumption, $\sqz(\mathcal{D},\Delta)$ is feasible, and so $\Pi_{T^\perp}$ is feasible in \eqref{eq.newprimal}; also, $(\tilde\gamma,\tilde{y})=(0,0)$ is feasible in \eqref{eq.newdual}.
By this feasibility and the extreme value theorem, we may take optimizers $X^\natural$ and $(\tilde{\gamma}^\natural,\tilde{Y}^\natural)$ of \eqref{eq.newprimal} and \eqref{eq.newdual}, respectively, and let $M$ and $(\gamma,Y)$ denote the corresponding feasible points of $\sqz(\mathcal{D},\Delta)$ and $\dual(\mathcal{D},\Delta)$, respectively.
Then the dual value of $(\alpha,Y)$ equals the primal value of $M$:
\begin{align*}
\Delta^2\sum_{z\in\mathcal{Z}(\mathcal{D})}\gamma(z)-\operatorname{tr}Y
&=\Delta^2\sum_{z\in\mathcal{Z}_2}\tilde{\gamma}(z)+\frac{\Delta^2}{\lambda}|\mathcal{Z}_1|-\bigg(\sum_{z\in\mathcal{Z}_2}\tilde{\gamma}^\natural(z)\|\Pi_Tz\|^2+\frac{\Delta^2}{\lambda}|\mathcal{Z}_1|-\operatorname{dim}T+\operatorname{tr}\tilde{Y}^\natural\bigg)\\
&=\operatorname{val}\eqref{eq.newdual}+\operatorname{dim}T
=\operatorname{val}\eqref{eq.newprimal}+\operatorname{dim}T
=\operatorname{tr}(\Pi_{T^\perp}X^\natural\Pi_{T^\perp})+\operatorname{tr}\Pi_T
=\operatorname{tr}M.
\end{align*}
Combining this with weak duality gives
\[
\val\dual(\mathcal{D},\Delta)
\geq \Delta^2\sum_{z\in\mathcal{Z}(\mathcal{D})}\gamma(z)-\operatorname{tr}Y
=\operatorname{tr}M
\geq\val\sqz(\mathcal{D},\Delta)
\geq\val\dual(\mathcal{D},\Delta).
\]
Considering the far left- and right-hand sides, we may conclude the desired equality.
\end{proof}

\begin{theorem}
\label{thm.strong duality fixed contact span}
Suppose $\mathcal{D}=\{(x_i,y_i)\}_{i\in\mathcal{I}}$ is $\Delta$-fixed.
Then $\sqz(\mathcal{D},\Delta)$ admits a Haar dual certificate if and only if the contact vectors of $\mathcal{D}$ span $\sspan\{x_i\}_{i\in\mathcal{I}}$.
\end{theorem}

\begin{proof}
($\Leftarrow$)
Select any finite collection $\mathcal{Z}_0$ of contact vectors of $\mathcal{D}$ that span $\sspan\{x_i\}_{i\in\mathcal{I}}$, and put $X:=\sum_{z\in\mathcal{Z}_0}zz^\top$.
Let $\lambda$ be the smallest non-zero eigenvalue of $X$, and define
\[
\gamma(z)
:=\left\{\begin{array}{cl}
1/\lambda&\text{if }z\in\mathcal{Z}_0\\
0&\text{if }z\in\mathcal{Z}(\mathcal{D})\setminus\mathcal{Z}_0
\end{array}\right\},
\qquad
Y:=\frac{1}{\lambda}X-\Pi,
\]
where $\Pi$ denotes orthogonal projection onto $\sspan\{x_i\}_{i\in\mathcal{I}}$.
It is straightforward to check that $(\gamma,Y)$ is feasible in $\dual(\mathcal{D},\Delta)$ with objective value $\tr\Pi$.
By Lemma~\ref{lem.ep.vectors}(i) and weak duality, $(\gamma,Y)$ is therefore a Haar dual certificate.

($\Rightarrow$)
Let $\Pi_\gamma$ denote orthogonal projection onto the column space of $\sum_{z\in\mathcal{Z}(\mathcal{D})}\gamma(z)zz^\top$.
We first claim that if $(\gamma,Y)$ is feasible in $\dual(\mathcal{D},\Delta)$, then so is $(\gamma,\Pi_\gamma Y\Pi_\gamma)$, and with monotonically larger objective value.
Indeed, feasibility follows from
\[
I-\sum_{z\in\mathcal{Z}(\mathcal{D})}\gamma(z)zz^\top+\Pi_\gamma Y\Pi_\gamma
=(I-\Pi_\gamma)+\Pi_\gamma\bigg(I-\sum_{z\in\mathcal{Z}(\mathcal{D})}\gamma(z)zz^\top+Y\bigg)\Pi_\gamma
\succeq 0,
\]
and the objective value is monotonically larger since $\tr(\Pi_\gamma Y\Pi_\gamma)=\tr(Y\Pi_\gamma)\leq\tr(Y)$, where the last step follows from the von Neumann trace inequality.
As such, the assumed Haar dual certificate $(\gamma,Y)$ satisfies $\operatorname{Col}(Y)\subseteq\operatorname{Col}(\Pi_\gamma)$ without loss of generality.

Put $Q:=\sum_{z\in\mathcal{Z}(\mathcal{D})}\gamma(z)\frac{\Delta^2}{\|z\|^2}zz^\top-Y$.
Then $\tr Q$ equals the dual value of $(\gamma,Y)$.
Let $\Pi$ denote orthogonal projection onto $\sspan\{x_i\}_{i\in\mathcal{I}}$.
Then $\operatorname{Col}(Y)\subseteq\operatorname{Col}(\Pi_\gamma)\subseteq\operatorname{Col}(\Pi)$, and so we may strengthen an inequality that is implied by the dual feasibility of $(\gamma,Y)$:
\[
A
:=\Pi-\sum_{z\in\mathcal{Z}(\mathcal{D})}\gamma(z)\Big(1-\frac{\Delta^2}{\|z\|^2}\Big)zz^\top-Q
=\Pi-\sum_{z\in\mathcal{Z}(\mathcal{D})}\gamma(z)zz^\top+Y
\succeq0.
\]
By Lemma~\ref{lem.ep.vectors}(i), $\Pi\in\arg\sqz(\mathcal{D},\Delta)$, and so $\tr\Pi=\tr Q$.
As such,
\[
0
\leq \tr A
=-\sum_{z\in\mathcal{Z}(\mathcal{D})}\gamma(z)(\|z\|^2-\Delta^2)
\leq 0.
\]
Considering the far left- and right-hand sides, we infer two important conclusions:
\begin{itemize}
\item[(a)]
Since $A\succeq0$ and $\tr A=0$, we necessarily have $A=0$.
\item[(b)]
Since $\displaystyle \sum_{z\in\mathcal{Z}(\mathcal{D})}\gamma(z)(\|z\|^2-\Delta^2)=0$, then $\gamma(z)>0$ only if $z$ is a contact vector of $\mathcal{D}$.
\end{itemize}
To be explicit, (b) applies Lemma~\ref{lem.fixed contact length}.
Let $T$ denote the span of the contact vectors of $\mathcal{D}$.
Rearranging $A=0$ from (a) gives $\sum_{z\in\mathcal{Z}(\mathcal{D})}\gamma(z)zz^\top=\Pi+Y$, and so (b) implies
\[
\sspan\{x_i\}_{i\in\mathcal{I}}
=\operatorname{Col}(\Pi)
\subseteq T
\subseteq\sspan\{x_i\}_{i\in\mathcal{I}},
\]
meaning $T=\sspan\{x_i\}_{i\in\mathcal{I}}$, as desired.
\end{proof}

\begin{lemma}[complementary slackness]
\label{lem.compslack}
Suppose $\sqz(\mathcal{D},\Delta)$ admits a Haar dual certificate and select any $M\in\arg\sqz(\mathcal{D},\Delta)$.
Then $\arg\dual(\mathcal{D},\Delta)$ is the set of points $(\gamma,Y)$ that are feasible in $\dual(\mathcal{D},\Delta)$ and further satisfy
\begin{align}
\label{eq.supp alpha}
\operatorname{supp}(\gamma)
&\subseteq\Big\{z\in\mathcal{Z}(\mathcal{D}):\|M^{1/2}z\|=\Delta\Big\},\\
\label{eq.col space Y}
\operatorname{Col}(Y)
&\subseteq\Big\{x:Mx=x\Big\},\\
\label{eq.M vs dual vars}
M
&=\sum_{z\in\mathcal{Z}(\mathcal{D})}\gamma(z)(M^{1/2}z)(M^{1/2}z)^\top-Y.
\end{align}
\end{lemma}

\begin{proof}
Suppose $(\gamma,Y)$ is feasible in $\dual(\mathcal{D},\Delta)$.
Then $I-\sum_{z\in\mathcal{Z}(\mathcal{D})}\gamma(z)zz^\top+Y
\succeq0$.
Multiplying by $M^{1/2}$ on both sides then gives
\begin{equation}
\label{eq.pos def zero}
M-\sum_{z\in\mathcal{Z}(\mathcal{D})}\gamma(z)M^{1/2}zz^\top M^{1/2}+M^{1/2}YM^{1/2}
\succeq0.
\end{equation}
We take the trace and rearrange to get
\[
\operatorname{tr}M
\geq\sum_{z\in\mathcal{Z}(\mathcal{D})}\gamma(z)z^\top Mz-\operatorname{tr}(MY)
\geq\Delta^2\sum_{z\in\mathcal{Z}(\mathcal{D})}\gamma(z)-\operatorname{tr}(MY)
\geq\Delta^2\sum_{z\in\mathcal{Z}(\mathcal{D})}\gamma(z)-\operatorname{tr}Y,
\]
where the last step applied the von Neumann trace inequality.
Since $\sqz(\mathcal{D},\Delta)$ admits a Haar dual certificate by assumption, $\arg\dual(\mathcal{D},\Delta)$ is the set of points $(\gamma,Y)$ that are feasible in $\dual(\mathcal{D},\Delta)$ and further make all of the above inequalities achieve equality.

Equality in the second inequality is characterized by \eqref{eq.supp alpha}, while equality in the third inequality is characterized by \eqref{eq.col space Y}.
Equality in the first inequality occurs precisely when the positive-semidefinite matrix in \eqref{eq.pos def zero} has trace zero, i.e., the matrix equals zero.
As such,
\[
M
=\sum_{z\in\mathcal{Z}(\mathcal{D})}\gamma(z)M^{1/2}zz^\top M^{1/2}-M^{1/2}YM^{1/2}.
\]
Furthermore, \eqref{eq.col space Y} implies $M^{1/2}YM^{1/2}=Y$, and so the above is equivalent to \eqref{eq.M vs dual vars}.
\end{proof}

In the finite case, strong duality is guaranteed by Theorem~\ref{thm.finite sd}.
Given $M\in\arg\sqz(\mathcal{D},\Delta)$, then Lemma~\ref{lem.compslack} enables a quick procedure to find a dual certificate for $M$.
Denote
\[
\mathcal{Z}_0:=\Big\{z\in\mathcal{Z}(\mathcal{D}):\|M^{1/2}z\|=\Delta\Big\},
\qquad
E:=\Big\{x:Mx=x\Big\},
\]
and consider the feasibility semidefinite program
\begin{align}
\label{eq.findcert}\text{find}\quad&(\gamma,Y)\\
\nonumber\text{subject to}\quad&\sum_{z\in\mathcal{Z}_0}\gamma(z)zz^\top-Y\preceq I,\quad\Pi_{E^\perp}Y\Pi_{E^\perp}=0,\\
\nonumber&M=\sum_{z\in\mathcal{Z}_0}\gamma(z)(M^{1/2}z)(M^{1/2}z)^\top-Y,\quad Y\succeq0, \quad \gamma\geq0
\end{align}
Importantly, solving \eqref{eq.findcert} is much faster than solving $\dual(\mathcal{D},\Delta)$ since $|\mathcal{Z}_0|\ll|\mathcal{Z}(\mathcal{D})|$, and so \eqref{eq.findcert} can be used to promote any heuristic SqueezeFit solver to a fast certifiably correct algorithm, much like~\cite{Bandeira:16,IguchiMPV:17,RosenCBL:16}.
Indeed, given a prospective solution $M$ satisfying $0\preceq M\preceq I$, we may:
\begin{itemize}
\item[(i)]
Find the shortest vectors in $\{M^{1/2}z\}_{z\in\mathcal{Z}(\mathcal{D})}$ from $\{M^{1/2}x_i\}_{i\in\mathcal{I}}$.
If these shortest vectors have length $\Delta$, then this certifies that $M$ is feasible in $\sqz(\mathcal{D},\Delta)$ and gives $\mathcal{Z}_0$ for the next step.
\item[(ii)]
Solve the feasibility semidefinite program \eqref{eq.findcert} to find a dual certificate $(\gamma,Y)$.
\end{itemize}
By Lemma~\ref{lem.compslack}, the primal value of $M$ equals the dual value of $(\gamma,Y)$, and one may verify this \textit{a posteriori}.
Weak duality then implies $M\in\arg\sqz(\mathcal{D},\Delta)$.

One way to solve (i) is to partition $\mathcal{Z}(\mathcal{D})$ into subsets $\mathcal{Z}_{st}:=\{x_i-x_i:i,j\in\mathcal{I},y_i=s,y_j=t\}$ and then find the shortest vectors in $\{M^{1/2}z\}_{z\in\mathcal{Z}_{st}}$ from $\{M^{1/2}x_i\}_{i\in\mathcal{I}}$ for each $s,t\in[k]$ with $s\neq t$.
This amounts to a fundamental problem in computational geometry:
Given $A,B\in\mathbb{R}^d$, find the closest pairs $(x,y)\in A\times B$.
Litiu and Kountanis~\cite{LitiuK:97} devised an $O_d(n\log^{d-1}n)$ divide-and-conquer algorithm that solves the problem for the taxicab metric in the special case where $A$ and $B$ are linearly separable.
In practice, one might construct a $k$-d tree for $A$ in $O_d(n\log n)$ time, and then use it to perform nearest neighbor search in $O_d(\log n)$ time on average~\cite{FriedmanBF:76} for each member of $B$.
Next, (ii) is polynomial in $d$ and $|\mathcal{Z}_0|$, and furthermore, Lemma~\ref{lem.bound on omega M eps} below gives that $|\mathcal{Z}_0|\leq d^4$ for generic data (we suspect this upper bound is loose).
Overall, one may expect to accomplish (i) and (ii) in time that is roughly linear in $n$.

\begin{lemma}
\label{lem.bound on omega M eps}
For every $d,k,n>0$ and every $\{y_i\}_{i\in[n]}\in[k]^n$, there exists a set $\mathcal{X}$ that is open and dense in $(\mathbb{R}^d)^n$ such that for every $\{x_i\}_{i\in[n]}\in\mathcal{X}$, $\Delta>0$, and every $M\in\operatorname{sqz}(\{(x_i,y_i)\}_{i\in\mathcal{I}},\Delta)$,
\[
\Big|\Big\{(i,j):1\leq i<j\leq n,y_i\neq y_j,\|M^{1/2}(x_i-x_j)\|=\Delta\Big\}\Big|
<\Big(\tbinom{d+1}{2}+1\Big)^2.
\]
\end{lemma}

\begin{proof}
Fix $d,k,n>0$ and $\{y_i\}_{i\in[n]}\in[k]^n$, and denote $\Omega:=\{(i,j):1\leq i<j\leq n,y_i\neq y_j\}$ and $D:=\binom{d+1}{2}$.
We may assume $|\Omega|\geq(D+1)^2$, since the result is otherwise immediate.
In this case, we will find $\mathcal{X}$ for which something stronger than the desired conclusion holds:
For every $\{x_i\}_{i\in[n]}\in\mathcal{X}$ and every $\mathcal{J}\in\binom{\Omega}{(D+1)^2}$, there is no nonzero $L\in\mathbb{R}^{d\times d}$ such that $\|L(x_i-x_j)\|^2$ is constant over $(i,j)\in \mathcal{J}$.
In particular, for every such $\mathcal{J}$, we will find $\mathcal{K}\in\binom{\mathcal{J}}{D+1}$ for which there is no $L\in\mathbb{R}^{d\times d}$ such that $\|L(x_i-x_j)\|^2=1$ for every $(i,j)\in \mathcal{K}$.

We start by finding $\{z_l\}_{l\in[D+1]}\in(\mathbb{R}^d)^{D+1}$ for which there is no $L\in\mathbb{R}^{d\times d}$ such that $\|Lz_l\|^2=1$ for every $l\in[D+1]$.
Let $\{A_i\}_{i\in[D]}$ be any basis of the $D$-dimensional vector space of symmetric matrices.
By the spectral theorem, each $A_i$ can be decomposed as a linear combination of rank-$1$ matrices $\{v_{ij}v_{ij}^\top\}_{j\in[d]}$ of unit trace, and so $\{v_{ij}v_{ij}^\top\}_{i\in[D],j\in[d]}$ spans the vector space.
Select any basis from this spanning set, define the first $D$ of the $z_l$'s to be the corresponding $v_{ij}$'s, and take $z_{D+1}:=2z_1$.
Since $\{z_lz_l^\top\}_{l\in[D]}$ is a basis of unit-trace matrices, we have that $1=\|Lz_l\|^2=\langle L^\top L,z_lz_l^\top\rangle$ for every $l\in[D]$ if and only if $L^\top L=I$.
In this case, $\|Lz_{D+1}\|^2=4\|Lz_1\|^2=4$, and so there is no $L\in\mathbb{R}^{d\times d}$ such that $\|Lz_l\|^2=1$ for every $l\in[D+1]$.

Next, for each $\mathcal{K}\in\binom{\Omega}{D+1}$, we will construct a polynomial $p_\mathcal{K}\in\mathbb{R}[X_{ij}:i\in[d],j\in[n]]$ such that $p_\mathcal{K}(P)$ is nonzero only if there is no $L\in\mathbb{R}^{d\times d}$ such that $\|L(x_i-x_j)\|^2=1$ for every $(i,j)\in \mathcal{K}$.
To this end, select any basis $\{A_l\}_{l\in[D]}$ of the $D$-dimensional vector space of symmetric matrices.
For each $(i,j)\in\mathcal{K}$, consider the decomposition $(x_i-x_j)(x_i-x_j)^\top=\sum_{l\in[D]}c_{(i,j),l}A_l$, and let $F\in\mathbb{R}^{\mathcal{K}\times[D+1]}$ be defined by
\[
F_{(i,j),l}=\left\{\begin{array}{cl}
c_{(i,j),l}&\text{if }l\in[D]\\
1&\text{if }l=D+1.
\end{array}\right.
\]
We will take $p_\mathcal{K}$ to be the polynomial that maps $\{x_i\}_{i\in[n]}$ to $\operatorname{det}(F)$.
Indeed, if $\operatorname{det}(F)\neq0$, then the all-ones vector does not lie in the span of the first $D$ columns of $F$, i.e., there is no $L$ such that
\[
\|L(x_i-x_j)\|^2
=\langle L^\top L,(x_i-x_j)(x_i-x_j)^\top\rangle
=\sum_{l\in[D]}c_{(i,j),l}\langle L^\top L,A_l\rangle
=1
\qquad
\forall (i,j)\in\mathcal{K}.
\]
We claim that $p_\mathcal{K}$ is a nonzero polynomial provided there exists  $\{x_i\}_{i\in[n]}\in(\mathbb{R}^d)^n$ and a bijection $f\colon\mathcal{K}\to[D+1]$ such that $x_i-x_j=z_{f(i,j)}$, where $\{z_l\}_{l\in[D+1]}$ is the example constructed above.
Indeed, since $\{z_lz_l^\top\}_{l\in[D]}$ is a basis, the corresponding $D\times D$ block of $F$ has full rank, meaning the first $D$ columns of $F$ are linearly independent.
By construction, these columns are also independent of the all-ones vector, and so $\operatorname{det}(F)\neq0$.
Finally, since there exists $\{x_i\}_{i\in[n]}\in(\mathbb{R}^d)^n$ such that $p_\mathcal{K}(\{x_i\}_{i\in[n]})=\operatorname{det}(F)\neq0$, we must have $p_\mathcal{K}\neq0$.

We now use the polynomials $p_\mathcal{K}$ to construct a larger polynomial $p\in\mathbb{R}[X_{ij}:i\in[d],j\in[n]]$:
\[
p(X):=\prod_{\mathcal{J}\in\binom{\Omega}{(D+1)^2}}\sum_{\mathcal{K}\in\binom{\mathcal{J}}{D+1}}p_\mathcal{K}(X)^2.
\]
We claim that the result follows by taking $\mathcal{X}=\{\{x_i\}_{i\in[n]}:p(\{x_i\}_{i\in[n]})\neq0\}$.
By construction, we have that every $\{x_i\}_{i\in[n]}\in\mathcal{X}$ satisfies $p(\{x_i\}_{i\in[n]})\neq0$, meaning that for every $\mathcal{J}\in\binom{\Omega}{(D+1)^2}$, there exists $\mathcal{K}\in\binom{\mathcal{J}}{D+1}$ such that $p_\mathcal{K}(\{x_i\}_{i\in[n]})\neq0$ (implying there is no $L\in\mathbb{R}^{d\times d}$ such that $\|L(x_i-x_j)\|^2=1$ for every $(i,j)\in\mathcal{K}$).
It remains to establish that $\mathcal{X}$ is open and dense, i.e., that $p\neq0$, or equivalently, that for every $\mathcal{J}\in\binom{\Omega}{(D+1)^2}$, there exists $\mathcal{K}\in\binom{\mathcal{J}}{D+1}$ such that $p_\mathcal{K}\neq0$.

Pick $\mathcal{J}\in\binom{\Omega}{(D+1)^2}$ and consider the graph with vertex set $[n]$ and edge set $\mathcal{J}$.
This graph has $(D+1)^2$ edges, and so the main result in~\cite{Han:12} implies that either (i) there exists a vertex of degree at least $D+1$, or (ii) there exists a matching of size at least $D+1$.
In the case of (i), let $\mathcal{K}$ be any $D+1$ of the edges incident to the vertex of maximum degree.
Then $\mathcal{K}=\{(i,j):i\in\mathcal{K}'\}$ for some $j\in[n]$ and $\mathcal{K}'\in\binom{[n]\setminus\{j\}}{D+1}$.
In this case, we can take any bijection $f'\colon\mathcal{K'}\to[D+1]$ and define $\{x_i\}_{i\in[n]}$ so that $x_i=z_{f'(i)}$ whenever $i\in\mathcal{K}'$ and otherwise $x_i=0$ (here, $z_l$ is defined in the example above).
Then $f\colon\mathcal{K}\to[D+1]$ defined by $f(i,j)=f'(i)$ is also a bijection and $x_i-x_j=x_i=z_{f'(i)}=z_{f(i,j)}$ for every $(i,j)\in\mathcal{K}$, implying $p_\mathcal{K}\neq0$.
In the case of (ii), let $\mathcal{K}$ be any $D+1$ of the edges in the maximum matching, and define $\{x_i\}_{i\in[n]}$ as follows:
Select any bijection $f\colon\mathcal{K}\to[D+1]$, and for each $(i,j)\in\mathcal{K}$, define $x_i=z_{f(i,j)}$ and $x_j=0$.
(For each vertex $i\in[n]$ that is not incident to an edge in $\mathcal{K}$, we may take $x_i=0$, say.)
Then $x_i-x_j=x_i=z_{f(i,j)}$ for every $(i,j)\in\mathcal{K}$, and so $p_\mathcal{K}\neq0$.
In either case, we have that there exists $\mathcal{K}\in\binom{\mathcal{J}}{D+1}$ such that $p_\mathcal{K}\neq0$, as desired.
\end{proof}


\section{Projection factor recovery with SqueezeFit}
\label{sec.pfr with sf}

We observe that SqueezeFit frequently succeeds in projection factor recovery when the $T$-component of the data is ``well behaved'' and the $T^\perp$-component of the data is ``independent'' of the $T$-component.
See Figure~\ref{fig.comparison} for an illustrative example.
In this section, we prove two general instances of this phenomenon.
The first instance enjoys a short proof:

\begin{theorem}
\label{thm.pfr clean}
Given $\mathcal{D}_0=\{(x_i,y_i)\}_{i\in\mathcal{I}}$, then every $M\in\arg\sqz(\mathcal{D}_0,\Delta)$ satisfies $\operatorname{Col}(M)\subseteq\sspan\{x_i\}_{i\in\mathcal{I}}$.
Select any nonempty set $\mathcal{S}\subseteq(\sspan\{x_i\}_{i\in\mathcal{I}})^\perp$ and define $\mathcal{D}=\{(x_i+s,y_i)\}_{i\in\mathcal{I},s\in \mathcal{S}}$.
Then $\arg\sqz(\mathcal{D},\Delta)=\arg\sqz(\mathcal{D}_0,\Delta)$.
\end{theorem}

In particular, if the $T$-component $\mathcal{D}_0$ is ``well behaved'' (in the sense that every $M\in\arg\sqz(\mathcal{D}_0,\Delta$) satisfies $\operatorname{Col}(M)=\sspan\{x_i\}_{i\in\mathcal{I}}$), then SqueezeFit succeeds in projection factor recovery from $\mathcal{D}$ (just find any $M\in\arg\sqz(\mathcal{D},\Delta)$ and take $\Pi=M(M^\top M)^{-1}M^\top$).

\begin{proof}[Proof of Theorem~\ref{thm.pfr clean}]
Let $\Pi$ denote orthogonal projection onto $\sspan\{x_i\}_{i\in\mathcal{I}}$.
Then for every $M$ that is feasible in $\sqz(\mathcal{D}_0,\Delta)$, $\Pi M\Pi$ is also feasible with $\tr(\Pi M\Pi)=\tr(M\Pi)\leq\tr M$.
The last inequality follows from the von Neumann trace inequality, where equality occurs only if $\operatorname{Col}(M)\subseteq\operatorname{Col}(\Pi)$.
As such, $M\in\arg\sqz(\mathcal{D}_0,\Delta)$ only if $M=\Pi M\Pi$.
Next, $\mathcal{Z}(\mathcal{D}_0)\subseteq\mathcal{Z}(\mathcal{D})$, and so $\sqz(\mathcal{D}_0,\Delta)$ is a relaxation of $\sqz(\mathcal{D},\Delta)$.
Since every $\Pi M\Pi\in\arg\sqz(\mathcal{D}_0,\Delta)$ is trivially feasible in $\sqz(\mathcal{D},\Delta)$, we then have $\arg\sqz(\mathcal{D}_0,\Delta)\subseteq\arg\sqz(\mathcal{D},\Delta)\subseteq\arg\sqz(\mathcal{D}_0,\Delta)$.
\end{proof}

The above guarantee uses a weak notion of ``well behaved'' for the $T$-component of $\mathcal{D}$, but a strong notion of ``independent'' for the $T^\perp$-component.
In what follows, we strengthen ``well behaved'' to mean $\Delta$-fixed, and weaken ``independent'' so that the $T^\perp$-component isn't identical (but follows the same Gaussian distribution) as you vary the $T$-component.
We begin with a technical lemma, which requires a definition:
Given a closed convex cone $\mathcal{C}\subseteq\mathbb{R}^n$, the \textbf{statistical dimension} of $\mathcal{C}$ is given by
\[
\delta(\mathcal{C})
:=\mathop{\mathbb{E}}_{g\sim\mathcal{N}(0,I)}\Big(\sup_{x\in\mathcal{C}\cap\mathbb{S}^{n-1}}\langle x,g\rangle\Big)^2.
\]
The notion of statistical dimension was introduced in~\cite{living} to characterize phase transitions in compressed sensing and elsewhere.

\begin{lemma}
\label{lem.stat dim}
There exist universal constants $c_1,c_2,c_3>0$ for which the closed convex cone
\[
\mathcal{C}_n
~:=~\Big\{~v\in\mathbb{R}^n~:~v\geq0,~\max v\leq \frac{c_1\sqrt{\log n}}{n}\cdot1^\top v~\Big\}
\]
has statistical dimension $\delta(\mathcal{C}_n)\in[c_2n,\frac{1}{2}n]$ for every $n\geq c_3$.
\end{lemma}

\begin{theorem}
\label{thm.exact recovery living edge}
Let $c_1,c_2,c_3>0$ be the constants in Lemma~\ref{lem.stat dim}.
Suppose $\mathcal{D}_0=\{(x_i,y_i)\}_{i\in[a]}$ is $\Delta$-fixed, and let $\Pi$ denote orthogonal projection onto the $r$-dimensional $\sspan\{x_i\}_{i\in[a]}$ in $\mathbb{R}^d$.
For each $i\in[a]$, draw $\{g_{it}\}_{t\in[b]}$ independently from $\mathcal{N}(0,\sigma^2(I-\Pi))$, and define $\mathcal{D}=\{(x_i+g_{it},y_i)\}_{i\in[a],t\in[b]}$.
Then $\arg\sqz(\mathcal{D},\Delta)=\{\Pi\}$ with probability at least $1-6|\mathcal{Z}_0|e^{-c_2b/48}$, provided
\[
b\geq\max\bigg\{\frac{2}{c_2}(d-r),c_3\bigg\},
\qquad
\mathsf{SNR}\geq3c_1\Big(4+\frac{c_2}{2}\Big)\sqrt{\log b},
\]
where $\mathcal{Z}_0$ denotes the contact vectors of $\mathcal{D}_0$, $\lambda$ is the smallest non-zero eigenvalue of $\sum_{z\in\mathcal{Z}_0}zz^\top$, and $\mathsf{SNR}:=\lambda/(2r\sigma^2)$.
\end{theorem}

In order to appreciate the above definition of $\mathsf{SNR}$, first note from Theorem~\ref{thm.contact fixed} that the contact vectors of $\mathcal{D}_0$ contain whatever ``signal'' SqueezeFit uses to find $\Pi$.
In the idealized setting where the contact vectors form an orthogonal basis for $T=\sspan\{x_i\}_{i\in[a]}$ together with its negation, then $\lambda/2$ equals the squared length $\Delta^2$ of each contact vector.
Since this energy is spread over $r$ dimensions, we can say that the amount of signal per dimension is $\lambda/(2r)$.
Intuitively, if the contact vectors ``barely'' span (meaning $\lambda$ is small), then the signal is weaker, whereas additional contact vectors provide stronger signal.
Our notion of signal-to-noise ratio $\mathsf{SNR}$ compares the amount of signal per dimension of $T$ to the amount of noise per dimension of $T^\perp$.

The $\mathsf{SNR}$ threshold in Theorem~\ref{thm.exact recovery living edge} is tight up to logarithmic factors.
Specifically, if we fix $\mathsf{SNR}=\epsilon\in(0,2)$, then for every $\alpha\in(0,(1+\frac{\epsilon}{2})^{-1})$, there exists $d_0=d_0(\alpha)$ such that for every $d\geq d_0$, there exists $\Delta$-fixed $\mathcal{D}_0$ with $r=\lceil\alpha d\rceil$ such that $\arg\sqz(\mathcal{D},\Delta)=\{\Pi\}$ with probability $\geq1/2$ only if $b$ is superpolynomial in $d$.
To see this, let $\{e_i\}_{i\in[d]}$ denote the identity basis, put $a=r+1$, and define $\mathcal{D}_0$ by $(x_i,y_i)=(e_i,0)$ for $i\in[r]$ and $(x_{r+1},y_{r+1})=(0,1)$.
Then $\mathcal{D}_0$ is $\Delta$-fixed with $\Delta=1$ and contact vectors $\{\pm e_i\}_{i\in[r]}$, resulting in $\lambda=2$.
Now take $\sigma^2=1/(r\epsilon)$ so that $\mathsf{SNR}=\epsilon$, and construct $\mathcal{D}$ as in Theorem~\ref{thm.exact recovery living edge}.
By our assumptions on $\epsilon$ and $\alpha$, we may select $\beta\in(\frac{\epsilon\alpha}{2},\min\{\alpha,1-\alpha\})$, $p=\lfloor \beta d\rfloor$, and let $\Pi_p$ denote orthogonal projection onto $\sspan\{e_i\}_{i=r+1}^{r+p}$.
We claim that, unless $b$ is superpolynomial in $d$, then for sufficiently large $d$, it holds with probability $\geq1/2$ that $\Pi_p$ is feasible in $\sqz(\mathcal{D},\Delta)$.
Since $\tr \Pi_p=p<r=\tr\Pi$, we then have $\Pi\not\in\arg\sqz(\mathcal{D},\Delta)$.
Recall that if $Z$ is standard Gaussian and $Q$ is $\chi^2$-distributed with $q$ degrees of freedom, then
\begin{equation}
\label{eq.fundamental bounds}
\operatorname{Pr}(|Z|\geq \xi)\leq 2e^{-\xi^2/2},
\qquad
\operatorname{Pr}(|Q-q|\geq \xi)\leq 2e^{-c\min\{\xi,\xi^2/q\}},
\qquad
\xi\geq0
\end{equation}
for some universal constant $c>0$; in particular, these estimates follow from the Chernoff bound and Hanson--Wright inequality~\cite{RudelsonV:13}, respectively.
We apply these bounds to obtain the estimate
\begin{align*}
\Delta_p^2
&:=\min_{\substack{i\in[r]\\s,t\in[b]}}\Big\|\Pi_p\Big((x_i+g_{is})-(x_{r+1}+g_{r+1,t})\Big)\Big\|^2\\
&\geq 2\min_{\substack{i\in[a]\\s\in[b]}}\|\Pi_pg_{is}\|^2-2\Big(\max_{t\in[b]}\|\Pi_pg_{r+1,t}\|\Big)\Big(\max_{\substack{i\in[r]\\s,t\in[b]}}|\langle \Pi_pg_{is},\tfrac{\Pi_pg_{r+1,t}}{\|\Pi_pg_{r+1,t}\|}\rangle|\Big)\\
&\geq \sigma^2\cdot\Big(2(p-\xi_1)-2(p+\xi_1)^{1/2}\cdot\xi_2\Big)
\end{align*}
with probability $\geq1-2abe^{-c\min\{\xi_1,\xi_1^2/p\}}-2rb^2e^{-\xi_2^2/2}$.
If we select $\xi_1=p^{3/4}$ and $\xi_2=p^{1/4}$, then we may conclude $\Delta_p^2\geq2\beta/(\epsilon\alpha)-o_{d\to\infty}(1)$ with probability $\geq1/2$ unless $b$ is superpolynomial in $d$.
Since $\beta>\epsilon\alpha/2$, we have $\Delta_p^2>1=\Delta^2$ for large $d$, which implies $\Pi_p$ is feasible in $\sqz(\mathcal{D},\Delta)$.

Going the other direction, as a consequence of Theorem~\ref{thm.exact recovery living edge}, it holds that for every $\sigma\geq0$, there exists a sufficiently large $b$ such that $\arg\sqz(\mathcal{D},\Delta)=\{\Pi\}$ with high probability.
Indeed, set
\[
b_1:=\bigg\lceil\max\bigg\{\frac{2}{c_2}(d-r),c_3\bigg\}\bigg\rceil,
\qquad
\mathsf{SNR}_1:=3c_1\Big(4+\frac{c_2}{2}\Big)\sqrt{\log b_1},
\qquad
\sigma_1^2:=\frac{\lambda}{r\cdot\mathsf{SNR}_1}.
\]
If $\sigma\leq\sigma_1$, we may take $b=b_1$ by Theorem~\ref{thm.exact recovery living edge}.
Otherwise, set $p:=(\sigma_1/\sigma)^{d-r}$, and select $b$ large enough so that for independent Bernoulli random variables $\{B_{it}\}_{i\in[a],t\in[b]}$ with mean $p$, it holds with high probability that $\sum_{t\in[b]}B_{it}\geq b_1$ for every $i\in[a]$.
To see why this suffices, observe that there exists a distribution $\mathcal{F}$ such that 
\[
\mathcal{N}(0,\sigma^2(I-\Pi))
=p\cdot\mathcal{N}(0,\sigma_1^2(I-\Pi))+(1-p)\cdot\mathcal{F}.
\]
Let $\mathcal{D}_1$ denote the data points in $\mathcal{D}$ corresponding to the $\mathcal{N}(0,\sigma_1^2(I-\Pi))$ component.
In the high-probability event that $\sum_{t\in[b]}B_{it}\geq b_1$ for every $i\in[a]$, Theorem~\ref{thm.exact recovery living edge} gives that $\arg\sqz(\mathcal{D}_1,\Delta)=\{\Pi\}$ with high probability, which is also feasible in $\sqz(\mathcal{D},\Delta)$.
Since $\sqz(\mathcal{D}_1,\Delta)$ is a relaxation of $\sqz(\mathcal{D},\Delta)$, we may conclude that $\arg\sqz(\mathcal{D},\Delta)=\{\Pi\}$ in the same event.
(This argument takes $b$ to be at least $(\sigma/\sigma_1)^{d-r}$ times the lower bound in Theorem~\ref{thm.exact recovery living edge}, which is a dramatic increase even for $\sigma=2\sigma_1$ when $d-r$ is large.)

By contrast, such a large choice of $b$ is unnecessary for projection factor recovery to be information theoretically possible.
For example, it suffices to have $b>d-r$, even when $\sigma$ is arbitrarily large.
Indeed, for such $b$, it is straightforward to show that with probability $1$, every size-$b$ subcollection of $\{x_i+g_{it}\}_{i\in[a],t\in[b]}$ has affine rank $\geq d-r$, and the subcollections of affine rank $d-r$ are precisely those of the form $\{x_i+g_{it}:t\in[b]\}$.
As such, for projection factor recovery, it suffices to first find the unique balanced partition of the data that minimizes maximum affine rank, then apply principal component analysis to one of the resulting size-$b$ subcollections to recover $(\sspan\{x_i\}_{i\in[a]})^\perp$, and finally take $\Pi$ to be orthogonal projection onto the orthogonal complement of this subspace.
Interestingly, this method does not use the labels $\{y_i\}_{i\in[a]}$ to recover the projection.
Of course, this procedure is not computationally tractable, and it heavily exploits the model of the data.

\begin{proof}[Proof of Theorem~\ref{thm.exact recovery living edge}]
By Lemma~\ref{lem.ep.vectors}(i), $\arg\sqz(\mathcal{D}_0,\Delta)=\{\Pi\}$, which is trivially feasible in $\sqz(\mathcal{D},\Delta)$.
We will modify a dual certificate for $\sqz(\mathcal{D}_0,\Delta)$ to produce a point $(\gamma,Y)$ in $\dual(\mathcal{D},\Delta)$ of the same value.
By the proof of Theorem~\ref{thm.strong duality fixed contact span}, $\sqz(\mathcal{D}_0,\Delta)$ enjoys a dual certificate $(\tilde\gamma,\tilde{Y})$ of the form
\[
\tilde\gamma(z)
:=\left\{\begin{array}{cl}
1/\lambda&\text{if }z\in\mathcal{Z}_0\\
0&\text{if }z\in\mathcal{Z}(\mathcal{D}_0)\setminus\mathcal{Z}_0
\end{array}\right\},
\qquad
\tilde{Y}:=\frac{1}{\lambda}\sum_{z\in\mathcal{Z}_0}zz^\top-\Pi,
\]
Our choice of $(\gamma,Y)$ will take $Y=\tilde{Y}$, but selecting an appropriate $\gamma$ will be more delicate.
Denote $A_\gamma:=\sum_{z\in\mathcal{Z}(\mathcal{D})}\gamma(z)zz^\top-Y$ and $\Pi_\perp:=I-\Pi$.
We will select $\gamma\geq0$ such that
\begin{equation*}
\label{eq.sufficient conditions}
\|\Pi A_\gamma\Pi\|_{2\to2}\leq 1,
\quad
\|\Pi_\perp A_\gamma\Pi_\perp\|_{2\to2}\leq1,
\quad
\Pi A_\gamma\Pi_\perp=0,
\quad
\sum_{z\in\mathcal{Z}(\mathcal{D})}\gamma(z)=\sum_{z\in\mathcal{Z}(\mathcal{D}_0)}\tilde\gamma(z).
\end{equation*}
The first three above together imply $A_\gamma\preceq I$, thereby ensuring $(\gamma,Y)$ is feasible in $\dual(\mathcal{D},\Delta)$, while the final condition ensures that the value of $(\gamma,Y)$ in $\dual(\mathcal{D},\Delta)$ equals the value of $(\tilde\gamma,\tilde{Y})$ in $\dual(\mathcal{D}_0,\Delta)$, as desired.

We first claim that with high probability, it suffices to have the following:
For every $i,j\in[a]$ such that $x_i-x_j\in\mathcal{Z}_0$, there exists $v_{ij}\in\mathbb{R}^b$ such that
\begin{equation}
\label{eq.sufficient conditions 2}
v_{ij}\geq0,
\qquad
1^\top v_{ij}=\frac{1}{\lambda},
\qquad
\max v_{ij}\leq \frac{c_1\sqrt{\log b}}{b\lambda},
\qquad
G_{ij}v_{ij}=0,
\end{equation}
where $G_{ij}$ is the $d\times b$ matrix whose $t$th column is $g_{it}-g_{jt}$.
To see this, consider $\gamma$ defined by
\[
\gamma\Big((x_i+g_s)-(x_j+g_t)\Big)
=\left\{\begin{array}{cl}
(v_{ij})_t&\text{if }x_i-x_j\in\mathcal{Z}_0\text{ and }s=t\\
0&\text{otherwise.}
\end{array}\right.
\]
Then $\gamma\geq0$ is immediate, while $\|\Pi A_\gamma\Pi\|_{2\to2}\leq 1$ follows from
\[
\Pi A_\gamma\Pi
=\sum_{\substack{i,j\in[a]\\x_j-x_j\in\mathcal{Z}_0}}\sum_{t\in[b]}(v_{ij})_t(x_i-x_j)(x_i-x_j)^\top-\Pi Y\Pi
=\Pi\Big(\frac{1}{\lambda}X-\tilde{Y}\Big)\Pi
=\Pi.
\]
Next, \eqref{eq.col space Y} implies $\Pi_\perp Y=0$, and so
\begin{align*}
\|\Pi_\perp A_\gamma\Pi_\perp\|_{2\to2}
&=\bigg\|\sum_{\substack{i,j\in[a]\\x_j-x_j\in\mathcal{Z}_0}}\sum_{t\in[b]}(v_{ij})_t(g_{it}-g_{jt})(g_{it}-g_{jt})^\top\bigg\|_{2\to2}\\
&\leq\frac{c_1\sqrt{\log b}}{b\lambda}\sum_{\substack{i,j\in[a]\\x_j-x_j\in\mathcal{Z}_0}}\bigg\|\sum_{t\in[b]}(g_{it}-g_{jt})(g_{it}-g_{jt})^\top\bigg\|_{2\to2}\\
&\leq\frac{c_1\sqrt{\log b}}{b\lambda}\cdot r\cdot 2\sigma^2\cdot\Big(\sqrt{d-r}+2\sqrt{b}\Big)^2
\leq\frac{3c_1\sqrt{\log b}}{\mathsf{SNR}}\cdot\Big(\frac{d-r}{b}+4\Big)
\leq1,
\end{align*}
where the last line applies Corollary~5.35 in~\cite{Vershynin} with a union bound, the inequality $(p+q)^2\leq 3(p^2+q^2)$, and finally our assumptions on $b$ and $\mathsf{SNR}$; in particular, the first inequality in this last line is valid in an event $\mathcal{E}_1$ of probability $\geq1-2e^{-b/2}$.
Finally, we also have
\[
\Pi A_\gamma\Pi_\perp
=\sum_{\substack{i,j\in[a]\\x_j-x_j\in\mathcal{Z}_0}}\sum_{t\in[b]}(v_{ij})_t(x_i-x_j)(g_{it}-g_{jt})^\top
=\sum_{\substack{i,j\in[a]\\x_j-x_j\in\mathcal{Z}_0}}(x_i-x_j)(G_{ij}v_{ij})^\top
=0,
\]
and
\[
\sum_{z\in\mathcal{Z}(\mathcal{D})}\gamma(z)
=\sum_{\substack{i,j\in[a]\\x_j-x_j\in\mathcal{Z}_0}}\sum_{t\in[b]}(v_{ij})_t
=\frac{|\mathcal{Z}_0|}{\lambda}
=\sum_{z\in\mathcal{Z}(\mathcal{D}_0)}\tilde\gamma(z).
\]

It remains to find $v_{ij}$'s that satisfy \eqref{eq.sufficient conditions 2}.
Equivalently, we must show that with high probability, it holds that for every $i,j\in[a]$ such that $x_i-x_j\in\mathcal{Z}_0$, the random subspace $\operatorname{Null}(G_{ij})$ nontrivially intersects the cone $\mathcal{C}_b$ defined in Lemma~\ref{lem.stat dim}.
Importantly, each $\operatorname{Null}(G_{ij})$ has the same distribution as $\operatorname{Null}(G)$, where $G$ is $(d-r)\times b$ with independent standard Gaussian entries, meaning $\operatorname{Null}(G_{ij})$ is drawn uniformly from the Grassmannian of $(b-d+r)$-dimensional subspaces of $\mathbb{R}^b$.
Then Theorem~7.1 in~\cite{living} gives
\[
d-r\leq\delta(\mathcal{C}_b)-\xi
\qquad
\Longrightarrow
\qquad
\operatorname{Pr}\Big(\mathcal{C}_b\cap\operatorname{Null}(G_{ij})=\{0\}\Big)\leq4\operatorname{exp}\bigg(\frac{-\xi^2/8}{\delta(\mathcal{C}_b)+\xi}\bigg)
\]
whenever $\xi\geq0$.
Recalling Lemma~\ref{lem.stat dim} and selecting $\xi=c_2b/2$, then since $d-r\leq c_2b/2$ by assumption, we obtain nontrivial intersection between $\mathcal{C}_b$ and each $\operatorname{Null}(G_{ij})$ in an event $\mathcal{E}_2$ of probability $\geq1-4|\mathcal{Z}_0|e^{-c_2b/48}$.
The result then follows by a union bound between $\mathcal{E}_1$ and $\mathcal{E}_2$.
\end{proof}

\begin{proof}[Proof of Lemma~\ref{lem.stat dim}]
First, $\mathcal{C}_n$ is contained in the self-dual nonnegative orthant $\mathbb{R}_+^n$.
As such, Propositions~3.1 and~3.2 in~\cite{living} together give
\[
\delta(\mathcal{C}_n)
\leq\delta(\mathbb{R}_+^n)
=\frac{1}{2}n.
\]
Next, given any bounded set $T\subseteq\mathbb{R}^n$, let $N(T,\epsilon)$ denote the size of the largest $\epsilon$-packing, that is, the largest $S\subseteq T$ such that
\[
\min_{\substack{x,y\in S\\x\neq y}}\|x-y\|\geq\epsilon.
\]
By Sudakov minoration (see the proof of Theorem~7.4.1 in~\cite{Vershynin:18}, for example), we may bound statistical dimension in terms of packing numbers:
\begin{equation}
\label{eq.sudakov}
\delta(\mathcal{C}_n)
\geq\Big(\mathop{\mathbb{E}}_{g\sim\mathcal{N}(0,I)}\sup_{x\in\mathcal{C}_n\cap\mathbb{S}^{n-1}}\langle x,g\rangle\Big)^2
\geq c\cdot\sup_{\epsilon>0}\epsilon^2\log N(\mathcal{C}_n\cap\mathbb{S}^{n-1},\epsilon),
\end{equation}
where $c>0$ is some universal constant (one may take $c=1-e^{-1}$, for example).
As such, for the remaining lower bound, it suffices to estimate packing numbers of $\mathcal{C}_n\cap\mathbb{S}^{n-1}$.

To this end, we make a general observation:
Fix a measurable set $T\subseteq\mathbb{S}^{n-1}$ of normalized surface area $s(T):=\operatorname{area}(T)/\operatorname{area}(\mathbb{S}^{d-1})$, let $P$ denote a largest $\epsilon$-packing of $\mathbb{S}^{n-1}$, and find a rotation $Q_0\in\operatorname{SO}(n)$ that maximizes the cardinality of $Q_0P\cap T$.
If we draw $Q$ uniformly from $\operatorname{SO}(n)$, then the linearity of expectation gives
\begin{equation}
\label{eq.max vs avg}
N(T,\epsilon)
\geq |Q_0P\cap T|
\geq \mathbb{E}|QP\cap T|
=\mathbb{E}\sum_{x\in P}\mathbf{1}_{\{Qx\in T\}}
=\sum_{x\in P}\operatorname{Pr}(Qx\in T)
=s(T)\cdot N(\mathbb{S}^{n-1},\epsilon).
\end{equation}
A volume comparison argument gives the following estimate:
\begin{equation}
\label{eq.vol bound}
N(\mathbb{S}^{n-1},2\sin\tfrac{\theta}{2})
\geq(1+o_{n\to\infty}(1))\cdot\sqrt{2\pi n}\cdot\frac{\cos\theta}{\sin^{n-1}\theta}
\qquad
\forall \theta\in(0,\tfrac{\pi}{2}).
\end{equation}
(See~\cite{Jenssen} for a recent improvement to this estimate, and references therein for historical literature related to \eqref{eq.vol bound}, which will suffice for our purposes.)
With this, one obtains a lower bound on $N(T,\epsilon)$ by computing $s(T)$, which equals the probability that $g\sim\mathcal{N}(0,I)$ resides in the positively homogeneous set $\bigcup_{r>0}rT$ generated by $T$.

In our special case of $T=\mathcal{C}_n\cap\mathbb{S}^{n-1}$, we condition on the event $\{g\geq0\}$ to obtain
\[
s(\mathcal{C}_n\cap\mathbb{S}^{n-1})
=\operatorname{Pr}(g\in\mathcal{C}_n)
=2^{-n}\cdot\operatorname{Pr}\Big(\max h\leq\frac{c_1\sqrt{\log n}}{n}\cdot1^\top h\Big),
\]
where there coordinates of $h$ are independent with standard half-normal distribution.
Next, for every choice of $\alpha>0$, the union bound gives
\begin{align*}
\operatorname{Pr}\Big(\max h\leq\frac{c_1\sqrt{\log n}}{n}\cdot1^\top h\Big)
&\geq \operatorname{Pr}\Big(\max h<\alpha\sqrt{\log n}<\frac{c_1\sqrt{\log n}}{n}\cdot1^\top h\Big)\\
&\geq 1-\operatorname{Pr}\Big(\max h\geq\alpha\sqrt{\log n}\Big)-\operatorname{Pr}\Big(\frac{c_1}{n}1^\top h\leq\alpha\Big).
\end{align*}
We apply another union bound with \eqref{eq.fundamental bounds} to estimate the first term:
\[
\operatorname{Pr}\Big(\max h\geq\alpha\sqrt{\log n}\Big)
\leq n\cdot\operatorname{Pr}(|Z|\geq\alpha\sqrt{\log n})
\leq \frac{2}{n^{\alpha^2/2-1}}.
\]
To estimate the second term, we use the fact that each coordinate $h_i$ of $h$ has mean $\sqrt{2/\pi}$ and variance $1-2/\pi$ and apply Chebyshev's inequality:
\[
\operatorname{Pr}\Big(\frac{c_1}{n}1^\top h\leq\alpha\Big)
\leq\operatorname{Pr}\Bigg(\bigg|\frac{1}{n}\sum_{i\in[n]}h_i-\sqrt{\frac{2}{\pi}}\bigg|\geq\sqrt{\frac{2}{\pi}}-\frac{\alpha}{c_1}\Bigg)
\leq\frac{1-2/\pi}{(\sqrt{2/\pi}-\alpha/c_1)^2}.
\]
Combining these estimates, we may select $\alpha=4$ and $c_1=50$ (say) to get $s(\mathcal{C}_n\cap\mathbb{S}^{n-1})\geq\frac{1}{4}\cdot 2^{-n}$ for every $n\geq2$.
Then taking $\theta=\frac{\pi}{12}$ and combining with \eqref{eq.max vs avg} and \eqref{eq.vol bound} gives
\[
\log N(\mathcal{C}_n\cap\mathbb{S}^{n-1},2\sin\tfrac{\pi}{24})
\geq 0.65 n+\tfrac{1}{2}\log n-1.86+o_{n\to\infty}(1).
\]
By \eqref{eq.sudakov}, we are done. 
\end{proof}

\section{Numerical experiments}

\subsection{Implementation variants}

Before describing our numerical experiments, we first discuss a few different implementations of SqueezeFit that allow for scalability and robustness to outliers.

\textbf{Hinge loss.}
SqueezeFit is not feasible if $\Delta$ is larger than the minimum distance between two points with different labels.
In order to make SqueezeFit robust to outliers, we replace the $\Delta$ constraints in $\sqz(\mathcal{D},\Delta)$ with a hinge-loss penalization in the objective:
\begin{equation}
\text{minimize}
\quad
\tr M + \lambda \sum_{z \in \mathcal{Z}(\mathcal{D})} \Big( \Delta^2 - z^\top M z\Big)_+
\quad
\text{subject to}
\quad
0 \preceq M \preceq I
\tag{$\sqz_\lambda (\mathcal D, \Delta)$}
\end{equation}
Here, $a_+:=\max\{0,a\}$.
Importantly, $\sqz_\lambda (\mathcal D, \Delta)$ is feasible regardless of $\Delta$, but this comes at the price of an additional hyperparameter $\lambda>0$.

\textbf{Relaxing $\Delta$ constraints.}
Suppose $\mathcal{D}=\{(x_i,y_i)\}_{i\in[n]}$ is comprised of $n$ data points that are balanced over $k$ labels.
Then $\sqz(\mathcal{D},\Delta)$ is a semidefinite program with $d^2$ variables and $O_k(n^2)$ constraints.
When implemented directly, we find this program to be too slow once $n\geq200$, even when $d$ is small.
However, Lemma~\ref{lem.bound on omega M eps} indicates that typically, most of these $\Delta$ constraints are not tight, and so to accommodate larger values of $n$, we relax many of these constraints.
Fix $s\geq1$, let $S(i,\ell)$ denote the indices of the $s$ nearest neighbors to $x_i$ with label $\ell$, and put
\[
\mathcal{Z}_s(\mathcal{D})
:=\bigcup_{i\in[n]}\bigcup_{\substack{\ell\in[k]\\\ell\neq y_i}}\Big\{x_i-x_j:j\in S(i,\ell)\Big\}.
\]
Then replacing $\mathcal{Z}(\mathcal{D})$ in $\sqz(\mathcal{D},\Delta)$ with $\mathcal Z_s(\mathcal D)$ results in a relaxation with only $O(sn)$ constraints.
In practice, we obtain $\mathcal Z_s(\mathcal D)$ in $O_d(n\log n)$ time using a $k$-d tree.
Overall, passing to $\mathcal Z_s(\mathcal D)$ results in the following variants:
\begin{align}
&\text{minimize}
\quad
\tr M
\quad
\text{subject to}
\quad
z^\top Mz\geq\Delta^2~~\forall z\in\mathcal{Z}_s(\mathcal{D}),
\quad 
0 \preceq M \preceq I
\tag{$\sqz^s (\mathcal D, \Delta)$}\\
&\text{minimize}
\quad
\tr M + \lambda \sum_{z \in \mathcal{Z}_s(\mathcal D)} \Big( \Delta^2 - z^\top M z\Big)_+
\quad
\text{subject to}
\quad
0 \preceq M \preceq I
\tag{$\sqz_\lambda^s (\mathcal D, \Delta)$}
\end{align}
As one might expect, we observe that the optimizers of these variants are close approximations to optimizers of the original programs, even for moderate values of $s$, and the approximation is especially good when $\mathcal{D}$ exhibits clustering structure.

\textbf{Relaxing the identity constraint.}
For every $\mathcal{D}$, there exists $\delta=\delta(\mathcal{D})>0$ such that for every $\Delta\leq\delta$, it holds that every $M\in\arg\sqz(\mathcal{D},\Delta)$ satisfies $\tr M<1$, meaning the constraint $M\preceq I$ is not tight.
As such, we can afford to relax the identity constraint when $\Delta$ is small.
In the absence of the identity constraint, then up to scaling, we may equivalently put $\Delta=1$.
For this reason, we define the variant
\begin{equation}
\text{minimize}
\quad
\tr M
\quad
\text{subject to}
\quad
z^\top Mz\geq1~~\forall z\in\mathcal{Z}(\mathcal{D}),
\quad 
M \succeq 0
\tag{$\sqz(\mathcal D, 0^+)$}
\end{equation}
We similarly define $\sqz_\lambda(\mathcal{D},0^+)$, etc.
We observe that solving this relaxation is considerably faster.

\subsection{Handwritten digits}

In this subsection, we use SqueezeFit to perform compressive classification on the MNIST database of handwritten digits~\cite{LeCunCB:online}.
Specifically, we focus on binary classification between 4s and 9s, since these digits are easily confused.
There are 11,971 of such digits in the training set and 1,991 in the test set.
In order to apply SqueezeFit, we first decrease the dimensionality of the space by forming low-resolution versions of these digits in $[0,1]^{10\times10}$.
Next, we select $n$ points $\mathcal{D}$ at random from the training set and compute $M\in\arg\sqz_\lambda^s(\mathcal{D},0^+)$ with $\lambda=1$ and $s=5$.
We then apply $M^{1/2}$ to the entire training set.
Figure~\ref{fig.szq.pca} illustrates this SqueezeFit compression with $n=50$.
After compression, we apply $K$-nearest neighbor classification on the test set.
Figure~\ref{fig.mnist.image} illustrates this classification in the case of $n=800$ and $K=15$.
Table~\ref{table.mnist} compares the misclassification rates for $K$-nearest neighbor classification after compression with PCA, LDA, and SqueezeFit.
This comparison indicates that SqueezeFit is better at finding low-dimensional components that are amenable to classification, even though it was only able to see a fraction of the training set.

\begin{figure}
\begin{center}
\includegraphics[width=0.3\textwidth]{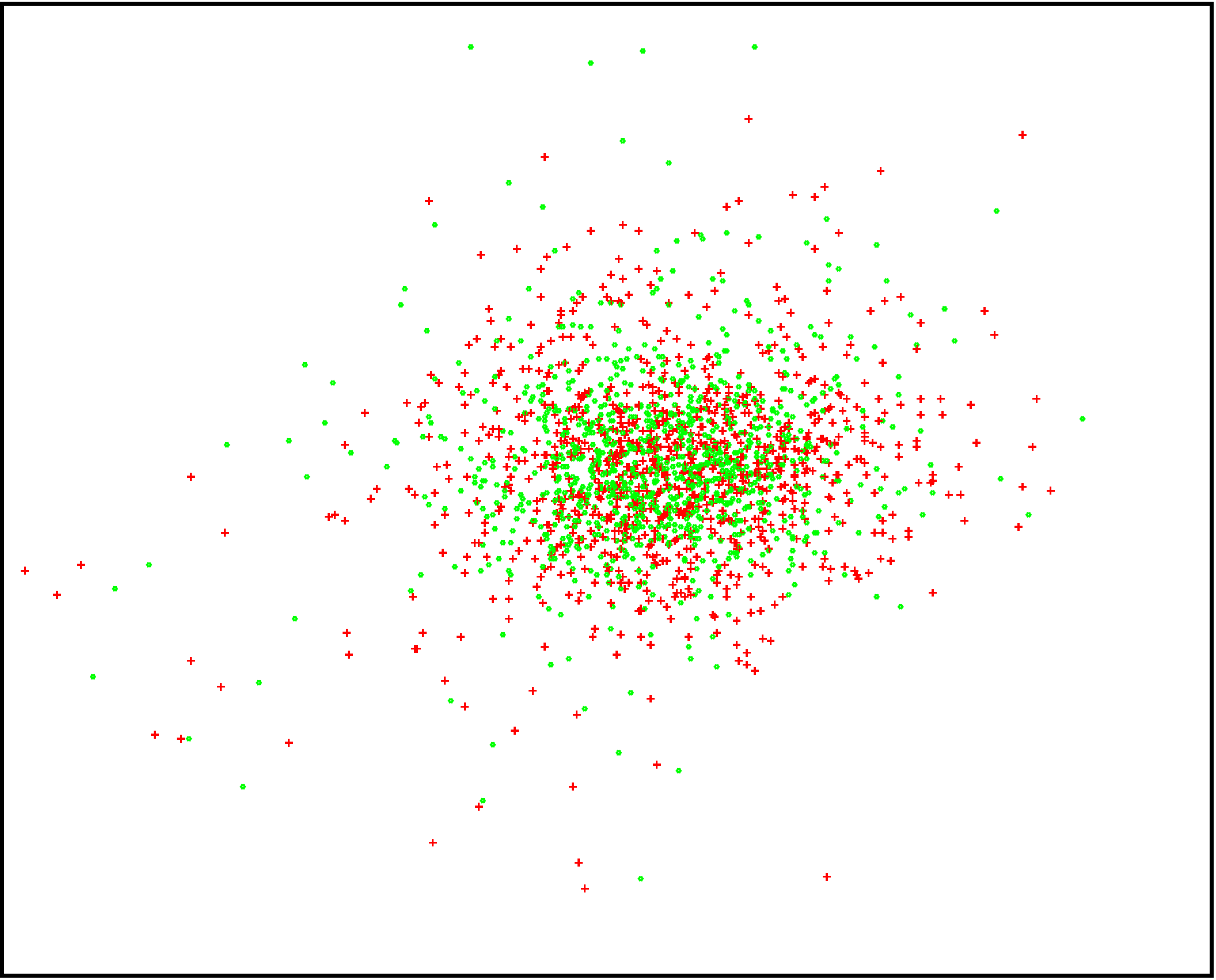}
\qquad
\includegraphics[width=0.3\textwidth]{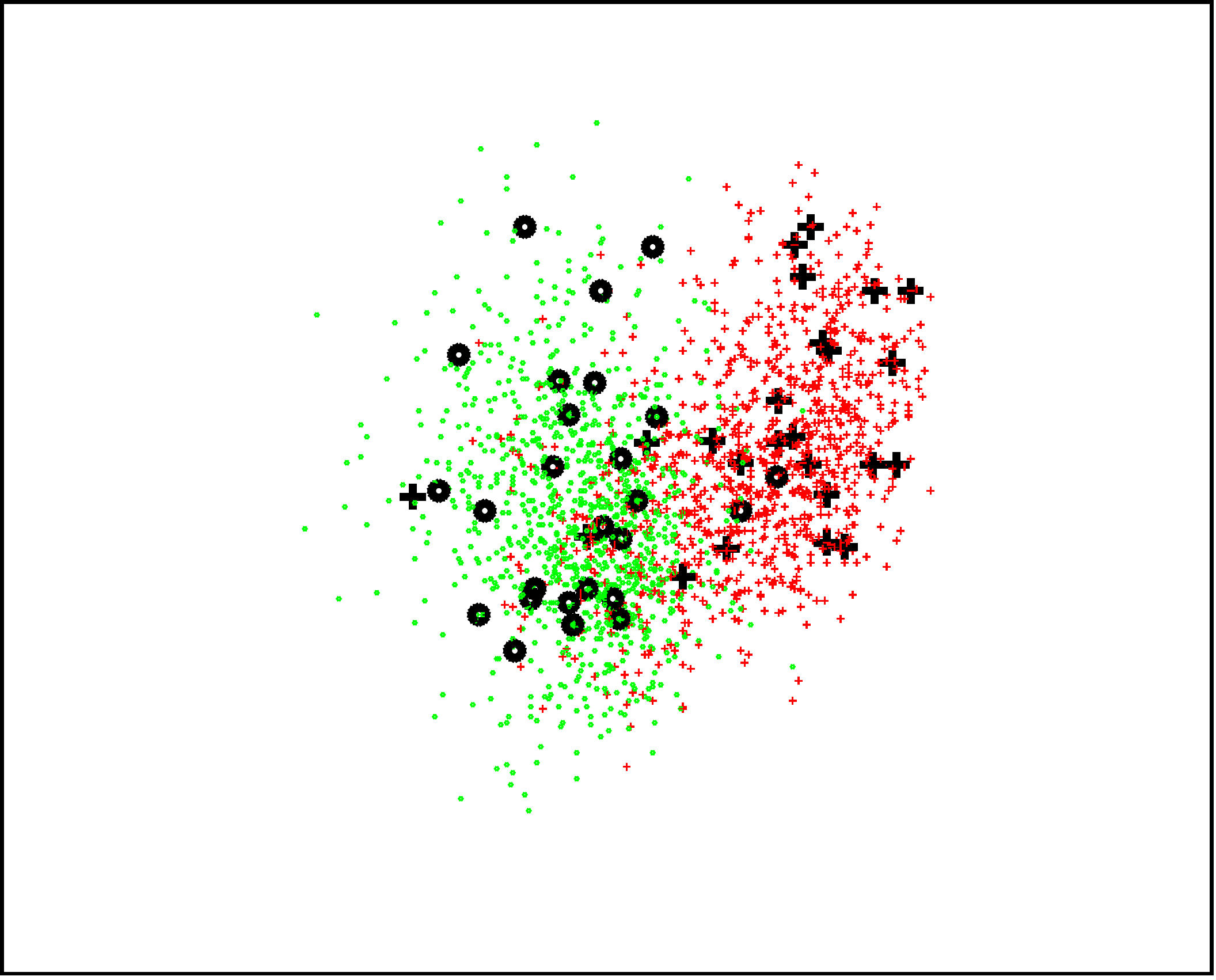}
\end{center}
\caption{\label{fig.szq.pca}
Take low-resolution versions of MNIST digits (4s and 9s) to obtain images in $[0,1]^{10\times10}$.
The two leading principal components of this data are displayed on the left.
Letting $\mathcal{D}$ denote $50$ of these low-resolution images, we run a variant of SqueezeFit to obtain $M$ of rank $5$.
Letting $\Pi$ denote orthogonal projection onto the span of the two leading eigenvectors of $M$, the right-hand plot illustrates $\Pi x$ for $2000$ randomly selected $x$'s from the low-resolution 4s and 9s.
The $50$ points from $\mathcal{D}$ are indicated with `+' and `o'.
Impressively, SqueezeFit identifies components that keep all of the 4s and 9s separated, despite only seeing a small sample.}
\end{figure}

\begin{figure}
\begin{center}
\includegraphics[height=0.25\textwidth, trim={2cm 1cm 2cm 1cm},clip]{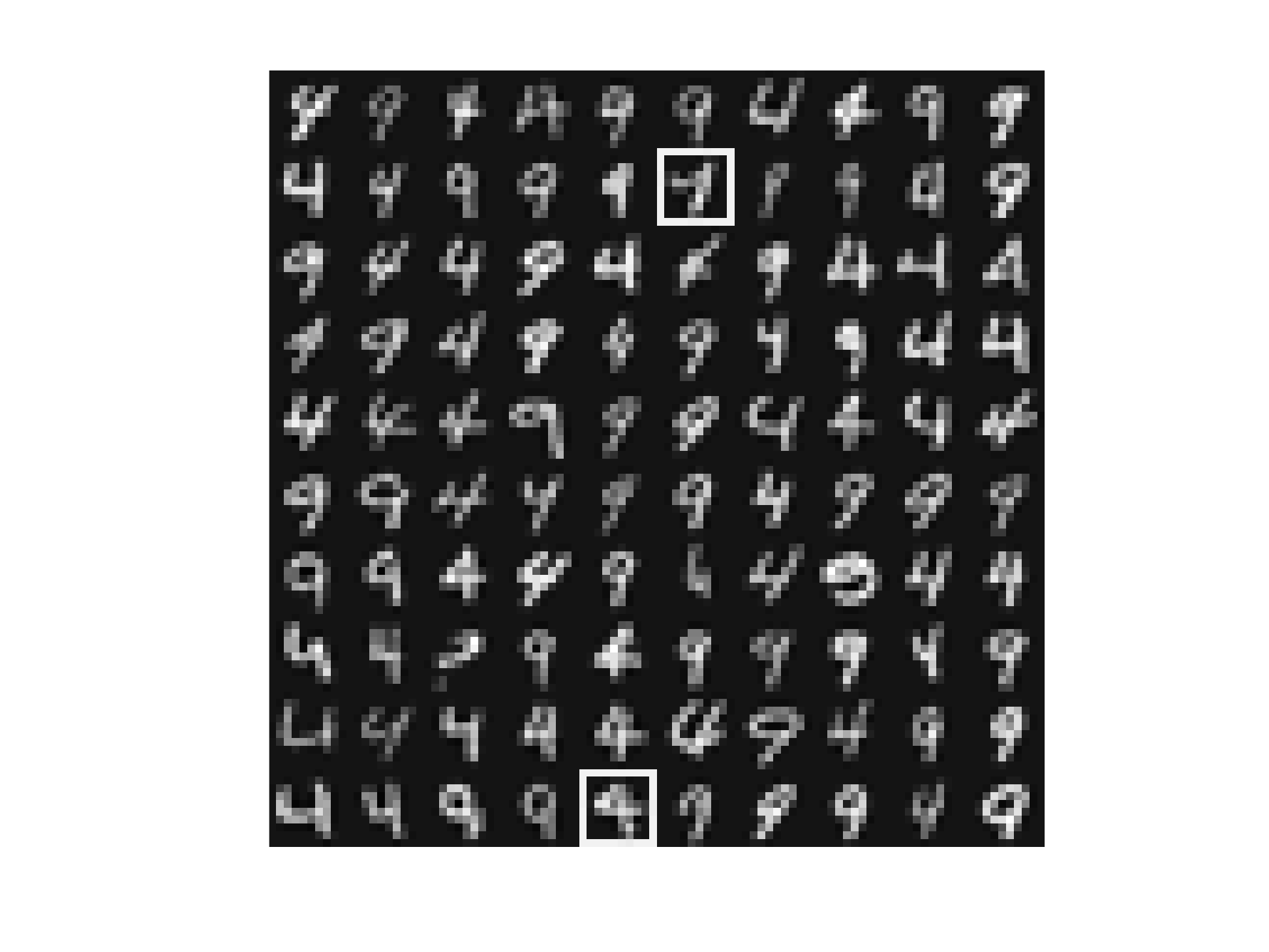}
\includegraphics[height=0.25\textwidth, trim={2cm 1cm 2cm 1cm},clip]{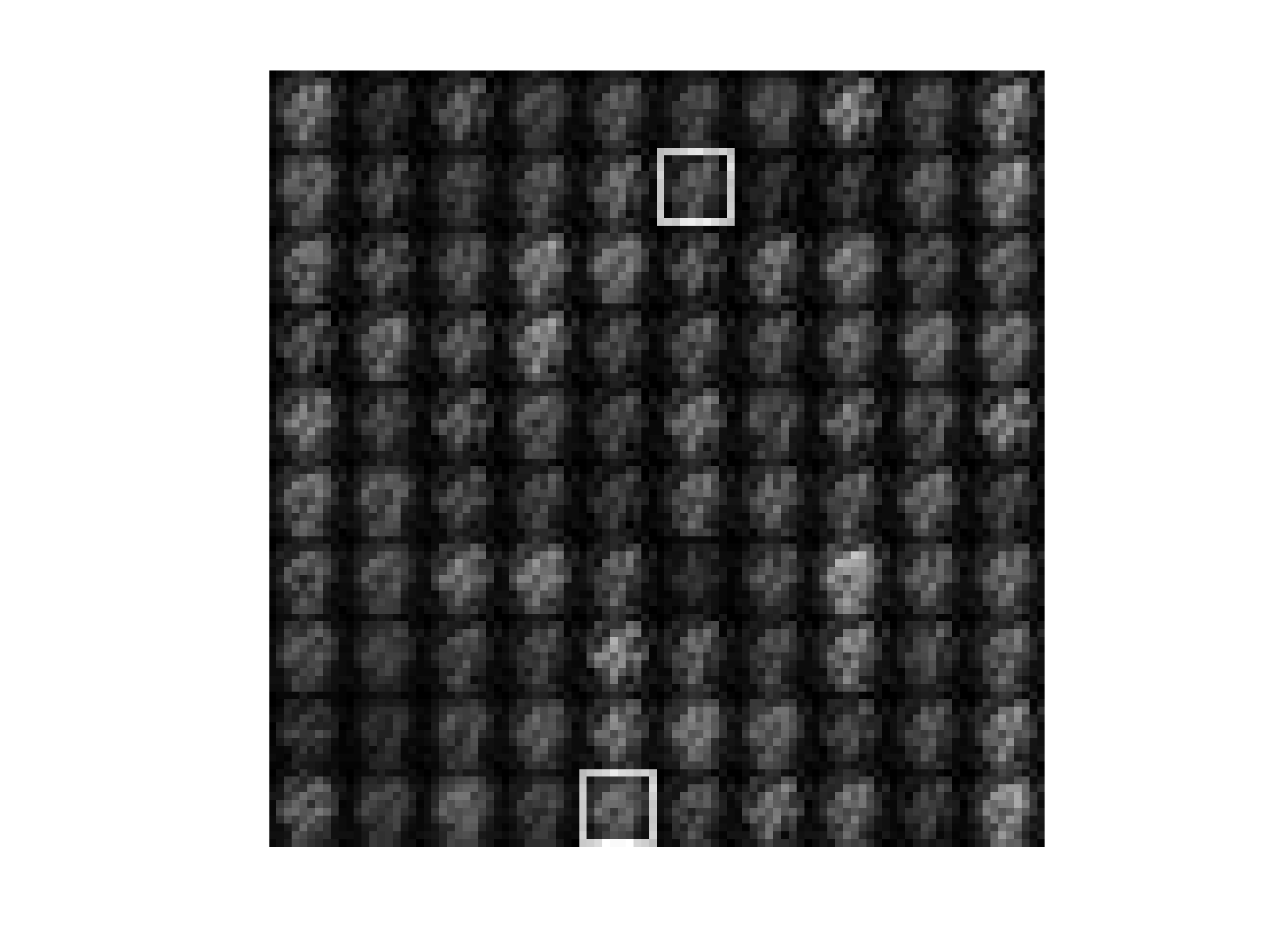}
\includegraphics[height=0.25\textwidth, trim={2cm 1cm 2cm 1cm},clip]{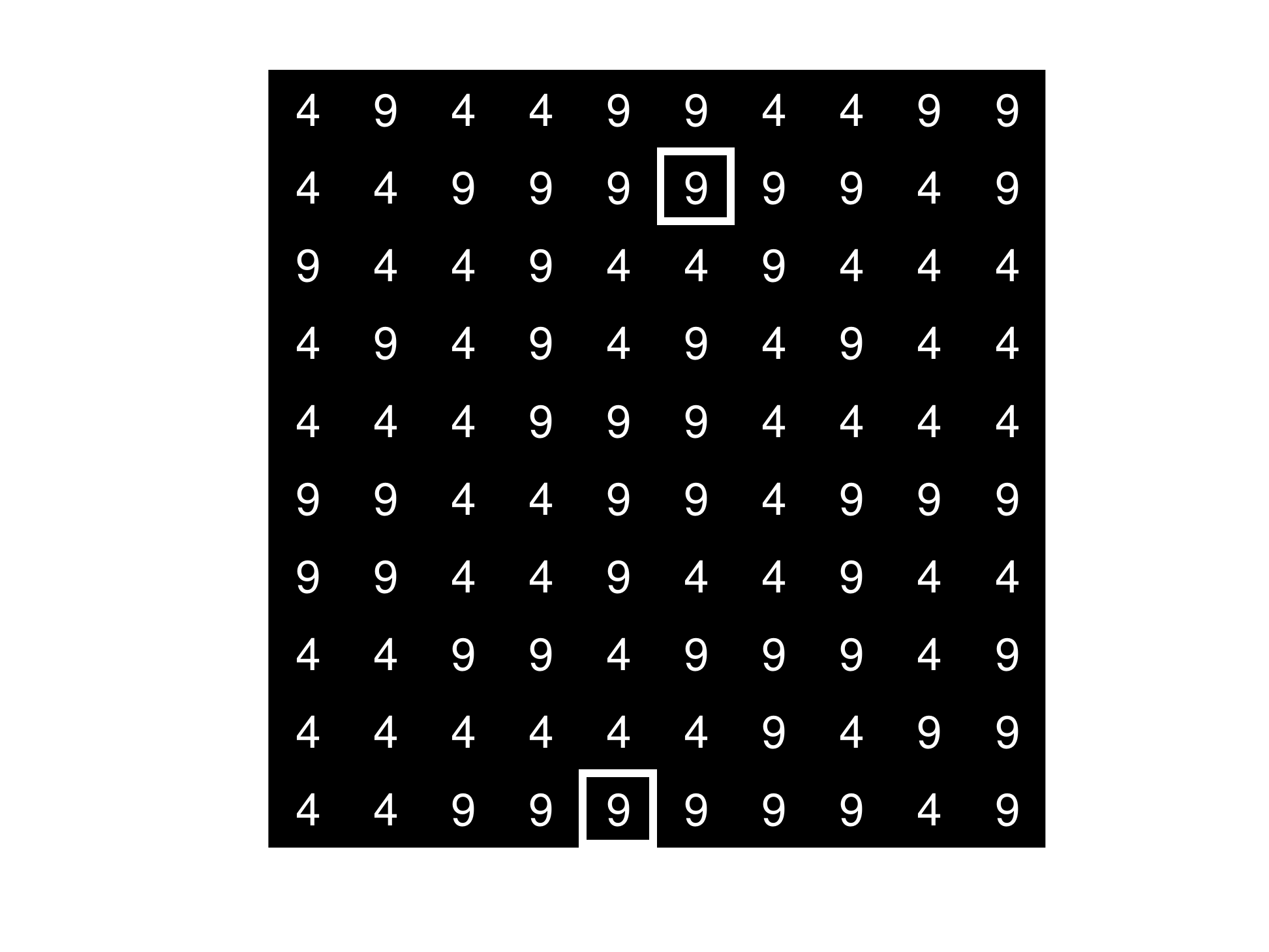}
\end{center}
\caption{ \label{fig.mnist.image}
\textbf{(left)}
Low-resolution versions of 4s and 9s from the MNIST test set.
\textbf{(middle)}
Given $800$ random members of the MNIST training set, run a variant of SqueezeFit to obtain $M$ of rank $11$, and then apply $M^{1/2}$ to the digits in the left panel.
\textbf{(right)} 
Apply $K$-nearest neighbor classification with $K=15$ to the compressed digits in the middle panel to predict labels.
Squares indicate misclassified images.}
\end{figure}

\begin{table}[t]
\begin{center}
\bigskip
\begin{tabularx}{\textwidth}{| >{\centering}X | >{\centering}X | >{\centering}X >{\centering}X | >{\centering}X >{\centering}X | >{\centering}X >{\centering\arraybackslash}X |}
\hline
& \textbf{Id} &  \multicolumn{2}{c|}{\textbf{PCA}} & \multicolumn{2}{c|}{\textbf{LDA}} & \multicolumn{2}{c|}{\textbf{SqueezeFit}}  \\ \hline\hline
$n$ & 0 & 11,791 & 11,791 & 800 & 11,791 & 50 &  800  \\
$r$ & 100 & 5 & 11 & 1 & 1 & 5 &  11  \\ \hline
$K=1$ & 2.15 &  16.62 & 5.47 & 7.78 & 6.07 & 5.97 &  4.01    \\
$K=5$ & 1.95 &  12.55 & 4.26 & 5.32 & 4.62 & 5.47 &  3.61    \\
$K=15$ & 2.26 &  12.00 & 4.11  & 5.32 & 3.87 & 5.27 &  3.87  \\ \hline
\end{tabularx}
\end{center}
\caption{\label{table.mnist}
Percentage of misclassified MNIST digits (4s and 9s) from $K$-nearest neighbor classification after compressing with either PCA, LDA, or SqueezeFit.
For each column, $n$ points from the training set were used to find a compression operator of rank $r$.
(The Id column takes identity to be the ``compression'' operator.)
This compression operator was then applied to the entire training set (of size 11,971) before running $K$-nearest neighbor classification on the test set (of size 1,991).
The LDA and SqueezeFit columns with $n=800$ computed a compression operator based on the same random sample of the training set.}
\end{table}

\subsection{Hyperspectral imagery}

The Indian Pines hyperspectral dataset~\cite{indianpines:online} consists of a $145\times145\times200$ data cube, representing a $145\times145$ overhead scene of farm land with $200$ different spectral reflectance bands ranging from $0.4$ to $2.5$ micrometers.
Each of the $145^2$ pixels in this scene is labeled by a member of $\{0,1,\ldots,16\}$; labels $1$ through $16$ either correspond to some sort of crop or some other material, whereas the label $0$ means the pixel is not labeled.
Since real-world hyperspectral data collection is slow, we wish to classify the contents of a pixel in a hyperspectral image from as few spectral measurements as possible.
For simplicity, we consider the binary classification task of distinguishing crops from non-crops.
In order to evaluate per-pixel compressive classification, we split the Indian Pines pixels with nonzero labels into 70\% training data and 30\% testing data.
Much like we did for MNIST digits above, we downsampled each $200$-dimensional feature vector into a point in $\mathbb{R}^{100}$.
We then applied PCA, LDA, and SqueezeFit to a subset of the training set in order to compute a compression operator, and we then applied this operator to the entire training set before performing $K$-nearest neighbor classification on the test set.
The results are summarized in Table~\ref{table.hyper}.
Much like Table~\ref{table.mnist}, this comparison indicates that SqueezeFit is better at finding low-dimensional components for classification.

\begin{table}[t]
\begin{center}
\bigskip
\begin{tabularx}{\textwidth}{| >{\centering}X | >{\centering}X | >{\centering}X >{\centering}X | >{\centering}X >{\centering}X | >{\centering}X >{\centering\arraybackslash}X |}
\hline
& \textbf{Id} &  \multicolumn{2}{c|}{\textbf{PCA}} & \multicolumn{2}{c|}{\textbf{LDA}} & \multicolumn{2}{c|}{\textbf{SqueezeFit}}  \\ \hline\hline
$n$ & 0 & 7,175 & 7,175 & 300 & 7,175 & 100 & 300  \\
$r$ & 100 & 3 & 5 & 1 & 1 & 3 &  5  \\ \hline
$K=1$ & 1.39 &  4.81 & 2.76 & 4.81 & 3.96 & 2.73 &  2.11    \\
$K=5$ & 1.69 &  4.35 & 2.79 & 3.12 & 3.02 & 2.83 &  2.27    \\
$K=15$ & 2.14 &  4.81 & 3.38 & 2.89 & 2.86 & 2.96 &  2.83  \\ \hline
\end{tabularx}
\end{center}
\caption{\label{table.hyper}
Percentage of misclassified pixels in the Indian Pines test set from $K$-nearest neighbor classification after compressing with either PCA, LDA, or SqueezeFit.
For each column, $n$ points from the training set were used to find a compression operator of rank $r$.
(The Id column takes identity to be the ``compression'' operator.)
This compression operator was then applied to the entire training set (of size 7,175) before running $K$-nearest neighbor classification on the test set (of size 3,074).
The LDA and SqueezeFit columns with $n=300$ computed a compression operator based on the same random sample of the training set.}
\end{table}

\section{Discussion}

In this paper, we introduced projection factor recovery as an idealization of the compressive classification problem, we proposed SqueezeFit as a semidefinite programming approach to this problem, and we provided theoretical guarantees for SqueezeFit in the context of projection factor recovery, as well as numerical experiments that compare SqueezeFit to alternative methods for compressive classification.
Through this investigation, the authors encountered a trove interesting research questions and opportunities for future work, which we discuss below.

First, under what conditions is projection factor recovery possible, both computationally and information theoretically?
In this paper, we focused on the case where SqueezeFit is well suited to perform projection factor recovery.
In particular, we established that the $\mathsf{SNR}$ threshold in Theorem~\ref{thm.exact recovery living edge} is tight up to logarithmic factors, but is the $\sqrt{\log b}$ factor necessary?
Importantly, the data model in Theorem~\ref{thm.exact recovery living edge} allows for exact projection factor recovery when $b$ is not too small (e.g., when $b>d-r$).
However, when $b=1$, exact projection factor recovery is no longer possible, although we observe approximate recovery in Figure~\ref{fig.comparison}.
Under what conditions does SqueezeFit give approximate recovery in this model?

While SqueezeFit has proven to be particularly amenable to theoretical investigation, our numerical experiments encountered a barrier to running the semidefinite program on large data.
For instance, it takes about 50 minutes to run SqueezeFit on a $100$-dimensional dataset comprised of $800$ data points (on a standard Macbook Air 2013).
This limitation forced us to randomly sample the training sets before running SqueezeFit (mimicking~\cite{MixonV:17}), but presumably, one can devise better sampling techniques based on some choice of leverage scores that account for the full geometry of the training set.
Without such sampling, then in order to handle datasets with more points in higher dimensions (e.g., the full MNIST training set~\cite{LeCunCB:online}), we require a different approach to solving \eqref{eq.min rank program}.
For example, one might consider a reformulation of \eqref{eq.min rank program} that optimizes over the $O(rd)$-dimensional Grassmannian of $r$-dimensional subspaces of $\mathbb{R}^d$ instead of the $\Omega(d^2)$-dimensional cone of positive semidefinite $d\times d$ matrices.
While such a formulation will be non-convex, there is a growing body of work that provides performance guarantees for such optimization problems~\cite{CandesLS:15,LiLSW:18,MaunuZL:17,BoumalVB:18}.
In fact, \cite{TorresaniL:07} proposes such a non-convex formulation of LMNN, and it performs well in practice.

Finally, there is room to further improve SqueezeFit for more effective compressive classification.
In practice, one will encounter application-specific design constraints on the sensing operator, and it would be interesting to incorporate these constraints into the SqueezeFit program.
For example, if the sensor $A$ is required to be a linear filter, then $M=A^\top A$ is diagonalized by the discrete Fourier transform, and so SqueezeFit reduces to a linear program.
Presumably, this constrained formulation enjoys runtime speedups over the original semidefinite program, along with refined performance guarantees.
Also, the numerical experiments in this paper focused on $k$-nearest neighbor classifiers in order to isolate the performance of dimensionality reduction alternatives.
However, the best known algorithms for image classification use convolutional neural networks (see~\cite{KrizhevskySH:12}, for example), and so it would be interesting to impose relevant convolution-friendly constraints in the SqueezeFit program to make use of this performance (see~\cite{LiPS:atd} for related work).

\section*{Acknowledgments}

CM and DGM were partially supported by the Center for Surveillance Research at the Ohio State University (an NSF industry/university cooperative research center).
DGM was partially supported by AFOSR FA9550-18-1-0107, NSF DMS 1829955, and the Simons Institute of the Theory of Computing.
SV was partially supported by EOARD FA9550-18-1-7007 and by the Simons Algorithms and Geometry (A\&G) Think Tank.
The views expressed in this article are those of the authors and do not reflect the official policy or position of the United States Air Force, Department of Defense, or the U.S.\ Government.


\begin{thebibliography}{WW}

\bibitem{living}
D.\ Amelunxen, M.\ Lotz, M.\ B.\ McCoy, J.\ A.\ Tropp,
Living on the edge:\ Phase transitions in convex programs with random data,
Inform.\ Inference 3 (2014) 224--294.

\bibitem{Bandeira:16}
A.\ S.\ Bandeira,
A note on probably certifiably correct algorithms,
C.\ R.\ Math.\ 354 (2016) 329--333.

\bibitem{BandeiraMR:17}
A.\ S.\ Bandeira, D.\ G.\ Mixon, B.\ Recht,
Compressive classification and the rare eclipse problem,
In: Compressed Sensing and its Applications, pp. 197--220. Birkh\"{a}user, Cham, 2017.

\bibitem{BoumalVB:18}
N.\ Boumal, V.\ Voroninski, A.\ S.\ Bandeira,
Deterministic guarantees for Burer-Monteiro factorizations of smooth semidefinite programs,
arXiv:1804.02008

\bibitem{BoydV:04}
S.\ Boyd, L.\ Vandenberghe,
Convex Optimization,
Cambridge, 2004.

\bibitem{CandesESV:15}
E.\ J.\ Cand\`{e}s, Y.\ C.\ Eldar, T.\ Strohmer, V.\ Voroninski,
Phase retrieval via matrix completion,
SIAM Rev.\ 57 (2015) 225--251.

\bibitem{CandesLS:15}
E.\ J.\ Cand\`{e}s, X.\ Li, M.\ Soltanolkotabi,
Phase retrieval via Wirtinger flow:\ Theory and algorithms,
IEEE Trans.\ Inform.\ Theory 61 (2015) 1985--2007.

\bibitem{CandesRT:06}
E.\ J.\ Cand\`{e}s, J.\ K.\ Romberg, T.\ Tao,
Stable signal recovery from incomplete and inaccurate measurements,
Comm.\ Pure Appl.\ Math.\ 59 (2006) 1207--1223.

\bibitem{CandesSV:13}
E.\ J.\ Cand\`{e}s, T.\ Strohmer, V.\ Voroninski,
Phaselift:\ Exact and stable signal recovery from magnitude measurements via convex programming,
Comm.\ Pure Appl.\ Math.\ 66 (2013) 1241--1274.

\bibitem{DavenportEtal:07}
M.\ A.\ Davenport, M.\ F.\ Duarte, M.\ B.\ Wakin, J.\ N.\ Laska, D.\ Takhar, K.\ F.\ Kelly, R.\ G.\ Baraniuk,
The smashed filter for compressive classification and target recognition,
Proc.\ SPIE, Computational Imaging 6498 (2007) 64980H.

\bibitem{Donoho:06}
D.\ L.\ Donoho,
Compressed sensing,
IEEE Trans.\ Inform.\ Theory, 52 (2006) 1289--1306.

\bibitem{FoucartR:13}
S.\ Foucart, H.\ Rauhut,
A mathematical introduction to compressive sensing,
Basel: Birkh\"{a}user, 2013.

\bibitem{FriedmanBF:76}
J.\ H.\ Friedman, J.\ L.\ Bentley, R.\ A.\ Finkel,
An algorithm for finding best matches in logarithmic time,
ACM Trans.\ Math.\ Software 3 (1976) 209--226.

\bibitem{semi-infinite}
M.\ A.\ Goberna, M.\ A.\ Lopez,
Linear semi-infinite programming theory: An updated survey,
Eur.\ J.\ Oper.\ Res.\ 143 (2002) 390--405.

\bibitem{Han:12}
Y.\ Han,
Tight bound for matching,
J.\ Comb.\ Optim.\ (2012) 23:322--330.

\bibitem{indianpines:online}
Hyperspectral Remote Sensing Scenes,
\url{ehu.eus/ccwintco/index.php/Hyperspectral_Remote_Sensing_Scenes}

\bibitem{IguchiMPV:17}
T.\ Iguchi, D.\ G.\ Mixon, J.\ Peterson, S.\ Villar,
Probably certifiably correct $k$-means clustering,
Math.\ Program.\ 165 (2017) 605--642.

\bibitem{Jenssen}
M.\ Jenssen, F.\ Joos, W.\ Perkins,
On kissing numbers and spherical codes in high dimensions,
arXiv:1803.02702

\bibitem{KrizhevskySH:12}
A.\ Krizhevsky, I.\ Sutskever, G.\ Hinton,
ImageNet Classification with Deep Convolutional Neural Networks,
NIPS 2012, 1097--1105

\bibitem{LeCunCB:online}
Y.\ LeCun, C.\ Cortes, C.\ J.\ Burges,
MNIST handwritten digit database,
\url{yann.lecun.com/exdb/mnist}

\bibitem{LiLSW:18}
X.\ Li, S.\ Ling, T.\ Strohmer, K.\ Wei,
Rapid, robust, and reliable blind deconvolution via nonconvex optimization,
Appl.\ Comput.\ Harmon.\ Anal.\ (2018).

\bibitem{LiPS:atd}
Y.\ Li, D.\ Pinckney, T.\ Strohmer,
Compressive Deep Learning,
in preparation.

\bibitem{LitiuK:97}
R.\ Litiu, D.\ I.\ Kountanis,
Closest Pair for Two Separated Sets of Points,
Congressus Numerantium (1997) 97--112.

\bibitem{LustigDP:07}
M.\ Lustig, D.\ Donoho, J.\ M.\ Pauly,
Sparse MRI:\ The application of compressed sensing for rapid MR imaging,
Magn.\ Reson.\ Med.\ 58 (2007) 1182--1195.

\bibitem{MaunuZL:17}
T.\ Maunu, T.\ Zhang, G.\ Lerman,
A well-tempered landscape for non-convex robust subspace recovery,
arXiv:1706.03896

\bibitem{MixonV:17}
D.\ G.\ Mixon, S.\ Villar,
Monte Carlo approximation certificates for $k$-means clustering,
arXiv:1710.00956

\bibitem{ReboredoRCR:16}
H.\ Reboredo, F.\ Renna, R.\ Calderbank, M.\ R.\ D.\ Rodrigues,
Bounds on the number of measurements for reliable compressive classification,
IEEE Trans.\ Signal Process.\ 64 (2016) 5778--5793.

\bibitem{RosenCBL:16}
D.\ M.\ Rosen, L.\ Carlone, A.\ S.\ Bandeira, J.\ J.\ Leonard,
SE-Sync:\ A certifiably correct algorithm for synchronization over the special Euclidean group,
arXiv:1612.07386

\bibitem{RudelsonV:13}
M.\ Rudelson, R.\ Vershynin,
Hanson-Wright inequality and sub-gaussian concentration,
Electronic Communications in Probability 18 (2013).

\bibitem{SidorenkoEtal:15}
P.\ Sidorenko, O.\ Kfir, Y.\ Shechtman, A.\ Fleischer, Y.\ C.\ Eldar, M.\ Segev, O.\ Cohen,
Sparsity-based super-resolved coherent diffraction imaging of one-dimensional objects,
Nature Comm.\ 6 (2015) 8209.

\bibitem{SilverEtal:16}
D.\ Silver, et al.,
Mastering the game of Go with deep neural networks and tree search, Nature 529 (2016) 484.

\bibitem{TorresaniL:07}
L.\ Torresani, K.-C.\ Lee,
Large margin component analysis,
NIPS 2007, 1385--1392.

\bibitem{Vershynin:18}
R.\ Vershynin,
High-dimensional probability: An introduction with applications in data science,
Cambridge, 2018.

\bibitem{Vershynin}
R.\ Vershynin,
Introduction to the non-asymptotic analysis of random
matrices,
In: Compressed Sensing, Theory and Applications, Y.\ Eldar,
G.\ Kutyniok (eds.),
Cambridge, 2012, 210--268.

\bibitem{WeinbergerS:09}
K.\ Q.\ Weinberger, L.\ K.\ Saul,
Distance metric learning for large margin nearest neighbor classification,
J.\ Mach.\ Learn.\ Res.\ 10 (2009) 207--244.


\end{thebibliography}
\end{document}